\documentclass[11pt,english]{article}
\usepackage{lmodern}

\usepackage[T1]{fontenc}
\usepackage[latin9]{inputenc}
\usepackage{geometry}
\usepackage{color}
\usepackage{babel}
\usepackage{float}
\usepackage{mathtools}
\usepackage{amsmath}
\usepackage{amsthm}
\usepackage{amssymb}
\usepackage{xargs}[2008/03/08]
\usepackage[unicode=true,pdfusetitle,
 bookmarks=true,bookmarksnumbered=false,bookmarksopen=false,
 breaklinks=false,pdfborder={0 0 1},backref=false,colorlinks=true]
 {hyperref}
\hypersetup{
 citecolor=blue, linkcolor=red}

\makeatletter

\floatstyle{ruled}
\newfloat{algorithm}{tbp}{loa}
\providecommand{\algorithmname}{Algorithm}
\floatname{algorithm}{\protect\algorithmname}

\theoremstyle{plain}
\newtheorem{thm}{\protect\theoremname}
  \theoremstyle{plain}
  \newtheorem{cor}{\protect\corollaryname}
  \theoremstyle{plain}
  \newtheorem{fact}{\protect\factname}
  \theoremstyle{plain}
  \newtheorem{prop}{\protect\propositionname}
  \theoremstyle{plain}
  \newtheorem{lem}{\protect\lemmaname}
  \theoremstyle{remark}
  \newtheorem*{rem*}{\protect\remarkname}
  \theoremstyle{remark}
  \newtheorem{claim}{\protect\claimname}

\usepackage{fancyheadings}
\usepackage[mathscr]{euscript}
\usepackage{amsfonts}
\usepackage{dsfont}
\usepackage{times}
\usepackage{amsmath}

\usepackage{txfonts}
\usepackage{appendix}
\allowdisplaybreaks

\theoremstyle{definition}
\newtheorem{mdl}{Model}

\makeatother

  \providecommand{\claimname}{Claim}
  \providecommand{\factname}{Fact}
  \providecommand{\lemmaname}{Lemma}
  \providecommand{\propositionname}{Proposition}
  \providecommand{\remarkname}{Remark}
\providecommand{\corollaryname}{Corollary}
\providecommand{\theoremname}{Theorem}

\begin{document}
\global\long\def\E{\mathbb{E}}
\global\long\def\P{\mathbb{P}}
\global\long\def\Var{\operatorname*{Var}}
\global\long\def\Cov{\operatorname*{Cov}}
\global\long\def\Tr{\operatorname*{Tr}}
\global\long\def\diag{\operatorname*{diag}}
\newcommandx\norm[2][usedefault, addprefix=\global, 1=\#1]{\Vert#1\|_{#2}}
\global\long\def\opnorm#1{\norm[#1]{\textup{op}}}
\global\long\def\rank#1{\operatorname*{rank}(#1)}
\global\long\def\md#1{\operatorname*{maxdiag}(#1)}
\global\long\def\indic{\operatorname*{\mathbb{I}}}
\global\long\def\diff{\operatorname{d}\!}
\global\long\def\argmax{\operatorname*{arg\,max}}
\global\long\def\argmin{\operatorname*{arg\,min}}
\global\long\def\Bern{\operatorname{Bern}}
\global\long\def\pos{\operatorname{pos}}
\global\long\def\round{\operatorname{round}}

\global\long\def\mtx#1{\bm{#1}}
\global\long\def\vct#1{\bm{#1}}
\global\long\def\real{\mathbb{R}}

\global\long\def\A{\mathbf{A}}
\global\long\def\Atilde{\widetilde{\A}}
\global\long\def\B{\mathbf{B}}
\global\long\def\D{\mathbf{D}}
\global\long\def\G{\mathbf{G}}
\global\long\def\H{\mathbf{H}}
\global\long\def\I{\mathbf{I}}
\global\long\def\J{\mathbf{J}}
\global\long\def\K{\mathbf{K}}
\global\long\def\L{\mathbf{L}}
\global\long\def\M{\mathbf{M}}
\global\long\def\P{\mathbb{P}}
\global\long\def\O{\mathbf{O}}
\global\long\def\Q{\mathbf{Q}}
\global\long\def\R{\mathbf{R}}
\global\long\def\U{\mathbf{U}}
\global\long\def\Uperp{\mathbf{U}_{\perp}}
\global\long\def\V{\mathbf{V}}
\global\long\def\Vtilde{\widetilde{\mathbf{V}}}
\global\long\def\Vprime{\mathbf{V}'}
\global\long\def\W{\mathbf{W}}
\global\long\def\Wtilde{\widetilde{\W}}
\global\long\def\X{\mathbf{X}}
\global\long\def\Y{\mathbf{Y}}
\global\long\def\Z{\mathbf{Z}}
\global\long\def\one{\mathbf{1}}
\global\long\def\calV{\mathcal{V}}
\global\long\def\calR{\mathcal{R}}
\global\long\def\calVrow{\mathcal{V}_{\mbox{row}}}
\global\long\def\calVcol{\mathcal{V}_{\mbox{col}}}

\global\long\def\Yhat{\widehat{\Y}}
\global\long\def\Ystar{\Y^{*}}
\global\long\def\Adj{\A}
\global\long\def\ObsAdj{\B}
\global\long\def\Noise{\W}
\global\long\def\HalfAdj{\boldsymbol{\Lambda}}
\global\long\def\HalfNoise{\boldsymbol{\Psi}}
\global\long\def\Yhp{\widehat{\Y}_{\perp}}
\global\long\def\Yround{\widehat{\Y}^{\text{R}}}
\global\long\def\CensorMat{\Z}

\global\long\def\yhat{\widehat{Y}}
\global\long\def\ystar{Y^{*}}
\global\long\def\adj{A}
\global\long\def\obsadj{B}
\global\long\def\noise{W}
\global\long\def\halfadj{\Lambda}
\global\long\def\halfnoise{\Psi}
\global\long\def\yround{\widehat{Y}^{\text{R}}}
\global\long\def\censormat{Z}

\global\long\def\OneMat{\J}
\global\long\def\onemat{J}
\global\long\def\onevec{\mathbf{1}}

\global\long\def\num{n}
\global\long\def\size{\ell}
\global\long\def\numclust{k}
\global\long\def\snr{s}
\global\long\def\inprob{p}
\global\long\def\outprob{q}
\global\long\def\flipprob{\epsilon}
\global\long\def\obsprob{\alpha}
\global\long\def\apxconst{\rho}
\global\long\def\error{\gamma}
\global\long\def\LabelStar{\boldsymbol{\sigma}^{*}}
\global\long\def\labelstar{\sigma^{*}}
\global\long\def\LabelHat{\widehat{\boldsymbol{\sigma}}}
\global\long\def\labelhat{\widehat{\sigma}}
\global\long\def\misrate{\operatorname*{\texttt{err}}}
\global\long\def\apkmedian{\texttt{\ensuremath{\apxconst}-kmed}}

\global\long\def\PT{\mathcal{P}_{T}}
\global\long\def\PTperp{\mathcal{P}_{T^{\perp}}}

\global\long\def\constsumthreeterms{72}
\global\long\def\constoperatornormW{182}
\global\long\def\t{{\displaystyle ^{\top}}}
\global\long\def\Abar{\bar{A}}
\global\long\def\bbar{\bar{b}}
\global\long\def\ybar{\bar{y}}
\global\long\def\xbar{\bar{x}}
\global\long\def\xtilde{\tilde{x}}
\global\long\def\constp{C_{\inprob}}
\global\long\def\consts{C_{\snr}}
\global\long\def\conste{C_{e}}
\global\long\def\constu{c_{u}}
\global\long\def\constdense{c_{d}}
\global\long\def\constgamma{C_{g}}
\global\long\def\constmis{C_{m}}

\global\long\def\pairset{\mathcal{L}}

\title{Exponential error rates of SDP for block models: \\Beyond Grothendieck's
inequality}

\author{Yingjie Fei and Yudong Chen\\School of Operations Research and Information
Engineering\\Cornell University\\\{yf275,yudong.chen\}@cornell.edu}

\date{}
\maketitle
\begin{abstract}
In this paper we consider the cluster estimation problem under the
Stochastic Block Model. We show that the semidefinite programming
(SDP) formulation for this problem achieves an error rate that decays
exponentially in the signal-to-noise ratio. The error bound implies
weak recovery in the sparse graph regime with bounded expected degrees,
as well as exact recovery in the dense regime. An immediate corollary
of our results yields error bounds under the Censored Block Model.
Moreover, these error bounds are robust, continuing to hold under
heterogeneous edge probabilities and a form of the so-called monotone
attack.

Significantly, this error rate is achieved by the SDP solution itself
without any further pre- or post-processing, and improves upon existing
polynomially-decaying error bounds proved using the Grothendieck\textquoteright s
inequality. Our analysis has two key ingredients: (i) showing that
the graph has a well-behaved spectrum, even in the sparse regime,
after discounting an exponentially small number of edges, and (ii)
an order-statistics argument that governs the final error rate. Both
arguments highlight the implicit regularization effect of the SDP
formulation.
\end{abstract}

\section{Introduction\label{sec:intro}}

In this paper, we consider the cluster/community\footnote{The words \emph{cluster} and \emph{community} are used interchangeably
in this paper.} estimation problem under the Stochastic Block Model (SBM) \cite{holland83}
with a growing number of clusters. In this model, a set of $\num$
nodes are partitioned into $\numclust$ unknown clusters of equal
size; a random graph is generated by independently connecting each
pair of nodes with probability~$\inprob$ if they are in the same
cluster, and with probability~$\outprob$ otherwise. Given one realization
of the graph represented by its adjacency matrix $\Adj\in\left\{ 0,1\right\} ^{\num\times\num}$,
the goal is to estimate the underlying clusters.

Much recent progress has been made on this problem, particularly in
identifying the sharp thresholds for exact/weak recovery when there
are a few communities of size linear in $\num$. Moving beyond this
regime, however, the understanding of the problem is much more limited
, especially in characterizing its behaviors with a growing (in $\num$)
number of clusters with sublinear sizes, and how the cluster errors
depend on the signal-to-noise ratio (SNR) in between the exact and
weak recovery regimes \cite{abbe2016recent,moore2017landscape}. We
focus on precisely these questions.\\

Let the ground-truth clusters be encoded by the \emph{cluster matrix}
$\Ystar\in\{0,1\}^{\num\times\num}$ defined by 
\[
\ystar_{ij}=\begin{cases}
\ensuremath{1} & \text{if nodes \ensuremath{i} and \ensuremath{j} are in the same community or if \ensuremath{i=j}},\\
\ensuremath{0} & \text{if nodes \ensuremath{i} and \ensuremath{j} are in different communities.}
\end{cases}
\]
We consider a now standard semidefinite programming (SDP) formulation
for estimating the ground-truth $\Ystar$:
\begin{equation}
\begin{aligned}\Yhat=\argmax_{\mathbf{Y}\in\mathbb{R}^{\num\times\num}}\; & \left\langle \Y,\Adj-\frac{\inprob+\outprob}{2}\OneMat\right\rangle \\
\mbox{s.t.}\; & \Y\succeq0,\;0\le\Y\le\J,\\
 & Y_{ii}=1,\forall i\in[\num],
\end{aligned}
\label{eq:SDP1}
\end{equation}
where $\OneMat$ is the $\num\times\num$ all-one matrix and $\left\langle \cdot,\cdot\right\rangle $
denotes the trace inner product. We seek to characterize the accuracy
of the SDP solution $\Yhat$ as an estimator of the true clustering.
Our main focus is the $\ell_{1}$ error $\norm[\Yhat-\Ystar]1$, where
$\norm[\M]1:=\sum_{i,j}\left|M_{ij}\right|$ denotes the entry-wise
$\ell_{1}$ norm. This $\ell_{1}$ error is a natural metric that
measures a form of pairwise cluster/link errors. In particular, note
that the matrix $\Ystar$ represents the pairwise cluster relationship
between nodes; an estimator of such is given by the matrix $\Yround\in\{0,1\}^{\num\times\num}$
obtained from rounding $\Yhat$ element-wise. The above $\ell_{1}$
error satisfies $\norm[\Yhat-\Ystar]1\ge\big\vert\{(i,j):\yround_{ij}\neq\ystar_{ij}\}\big\vert/2$,
and therefore upper bounds the number of pairs whose relationships
are incorrectly estimated by the SDP.

In a seminal paper~\cite{Guedon2015}, Guédon and Vershynin exhibited
a remarkable use of the Grothendieck's inequality, and obtained the
following high-probability error bound for the solution $\Yhat$ of
the SDP:
\begin{equation}
\frac{\norm[\Yhat-\Ystar]1}{\norm[\Ystar]1}\lesssim\sqrt{\frac{\numclust^{2}}{\text{SNR}\cdot n}},\label{eq:GV_error_bound}
\end{equation}
where $\text{SNR}=(\inprob-\outprob)^{2}/\left(\frac{1}{\numclust}\inprob+(1-\frac{1}{\numclust})\outprob\right)\approx(\inprob-\outprob)^{2}/\inprob$
is a measure of the signal-to-noise ratio. This bound holds even in
the sparse graph regime with $\inprob,\outprob=\Theta(1/\num)$, manifesting
the power of the Grothendieck's inequality.

In this paper, we go beyond the above results, and show that $\Yhat$
in fact satisfies (with high probability) the following exponentially-decaying
error bound
\begin{equation}
\frac{\norm[\Yhat-\Ystar]1}{\norm[\Ystar]1}\lesssim\exp\left[-\Omega\left(\text{SNR}\cdot\frac{\num}{\numclust}\right)\right]\label{eq:L1_error_bound}
\end{equation}
as long as $\text{SNR}\gtrsim\frac{\numclust^{2}}{\num}$ (Theorem~\ref{thm:exp_rate}).
The bound is valid in both the sparse and dense regimes. Significantly,
this error rate is achieved by the SDP~(\ref{eq:SDP1}) itself, without
the need of a multi-step procedure, even though we are estimating
a discrete structure by solving a continuous optimization problem.
In particular, the SDP approach does not require pre-processing of
graph (such as trimming and splitting) or an initial estimate of the
clusters, nor any non-trivial post-processing of~$\Yhat$ (such as
local cluster refinement or randomized rounding). 

If an explicit clustering of the nodes is concerned, the result above
also yields an error bound for estimating $\LabelStar$, the true
cluster labels. In particular, an explicit cluster labeling $\LabelHat$
can be obtained efficiently from $\Yhat$. Let $\misrate(\LabelHat,\LabelStar)$
denote the fraction of nodes that are labeled differently by $\LabelHat$
and $\LabelStar$ (up to permutation of the labels). This mis-classification
error can be shown to be upper bounded by the $\ell_{1}$ error $\norm[\Yhat-\Ystar]1/\norm[\Ystar]1$,
and therefore satisfies the same exponential bound (Theorem~\ref{thm:cluster_error_rate}):
\begin{equation}
\misrate(\LabelHat,\LabelStar)\lesssim\exp\left[-\Omega\left(\text{SNR}\cdot\frac{\num}{\numclust}\right)\right].\label{eq:cluster_error_bound}
\end{equation}

Specialized to different values of the errors, this single error bound
(\ref{eq:L1_error_bound}) implies sufficient conditions for achieving
exact recovery (strong consistency), almost exact recovery (weak consistency)
and weak recovery; see Section \ref{sec:related} for the definitions
of these recovery types. More generally, the above bound yields SNR
conditions sufficient for achieving any $\delta$ error. As to be
discussed in details in Section~\ref{sec:consequence}, these conditions
are (at least order-wise) optimal, and improve upon existing results
especially when the number of clusters $\numclust$ is allowed to
scale with $\num$. In addition, we prove that the above guarantees
for SDP are \emph{robust} against deviations from the standard SBM.
The same exponential bounds continue to hold in the presence of heterogeneous
edge probabilities as well as a form of monotone attack where an adversary
can modify the graph (Theorem~\ref{thm:exp_rate_robust}).  In addition,
we show that our results readily extend to the \emph{Censored Block
Model}, in which only partially observed data is available (Corollary~\ref{cor:censor}).

In addition to improved error bounds, our results also involve the
development of several new analytical techniques, as are discussed
below. We expect these techniques to be more broadly useful in the
analysis of SDP and other algorithms for SBM and related statistical
problems. 

\subsection{Technical highlights}

Our analysis of the SDP formulation builds on two key ingredients.
The first argument involves showing that the graph can be partitioned
into two components, one with a well-behaved spectrum, and the other
with an \emph{exponentially} small number of edges; cf.\ Proposition~\ref{prop:S2}.
Note that this partitioning is done \emph{in the analysis},\emph{
}rather than in the algorithm. It ensures that the SDP produces a
useful solution all the way down to the sparse regime with $\inprob,\outprob=\Theta(\frac{1}{\num})$.
The second ingredient is an order-statistics argument that characterizes
the interplay between the error matrix and the randomness in the graph;
cf.\ Proposition~\ref{prop:S1}. Upper bounds on the sum of the
top order statistics are what ultimately dictate the exponential decay
of the error. In both arguments, we make crucial use of the entry-wise
boundedness of the SDP solution $\Yhat$, which is a manifest of the
implicit regularization effect of the SDP formulation.

Our results are non-asymptotic in nature, valid for finite values
of $\num$; letting $\num\to\infty$ gives asymptotic results. All
other parameters $\inprob,\outprob$ and $\numclust$ are allowed
to scale arbitrarily with $\num$. In particular, the number of clusters
$\numclust$ may grow with $\num$, the clusters may have size sublinear
in $\num$, and the edge probabilities $\inprob$ and $\outprob$
may range from the sparse case $\Theta(\frac{1}{\num})$ to the densest
case $\Theta(1)$.  Our results therefore provide a general characterization
of the relationship between the SNR, the cluster sizes and the recovery
errors. This is particularly important in the regime of sublinear
cluster sizes, in which case all values of $\inprob$ and $\outprob$
are of interest. The price of such generality is that we do not seek
to obtain optimal values of the multiplicative constants in the error
bounds, doing which typically requires asymptotic analysis with scaling
restrictions on the parameters. In this sense, our results complement
recent work on the fundamental limits and sharp recovery thresholds
of SBM~\cite{abbe2016recent}. 

\subsection{Related work\label{sec:related}}

The SBM \cite{holland83,bollobas2004maxcut}, also known as the planted
partition model in the computer science community, is a standard model
for studying community detection and graph clustering. There is a
large body of work on the theoretical and algorithmic aspects of this
model; see for example \cite{ChenXu2016,abbe2016recent,mossel2012stochastic,abbe2015recovering}
and the references therein. Here we only briefly discuss the most
relevant work, and defer to Section~\ref{sec:main} for a more detailed
comparison after stating our main theorems.

Existing work distinguishes between several types of recovery \cite{abbe2016recent,decelle2011asymptotic},
including: (a) weak recovery, where the fraction of mis-clustered
nodes satisfies $\misrate(\LabelHat,\LabelStar)<1-\frac{1}{\numclust}$
and is hence better than random guess; (b) almost exact recovery (weak
consistency), where $\misrate(\LabelHat,\LabelStar)=o(1)$; (c) exact
recovery (strong consistency), where $\misrate(\LabelHat,\LabelStar)=0$.
The SDP relaxation approach to SBM has been studied in \cite{amini2014semidefinite,ames2014convex,chen2012clustering,ChenXu2016,oymak2011finding,krivelevich2006coloring,CMM,chen2012sparseclustering},
which mostly focus on exact recovery in the logarithmic-degree regime
$\inprob\gtrsim\frac{\log\num}{\num}$. Using the Grothendieck's inequality,
the work in \cite{Guedon2015} proves for the first time that SDP
achieves a non-trivial error bound in the sparse regime $\inprob\asymp\frac{1}{\num}$
with bounded expected degrees. In the two-cluster case, it is further
shown in \cite{MontanariSen16} that SDP in fact achieves the optimal
weak recovery threshold as long as the expected degree is large (but
still bounded). Our single error bound implies exact and weak recovery
in the logarithmic and bounded degree regimes, respectively. Our result
in fact goes beyond these existing ones and applies to every setting
in between the two extreme regimes, capturing the exponential decay
of error rates from $O(1)$ to zero. 

A very recent line of research aims to precisely characterize the
fundamental limits and phase transition behaviors of SBM \textemdash{}
in particular, what are the sharp SNR thresholds (including the leading
constants) for achieving the different recovery types discussed above.
When the number $\numclust$ of clusters is bounded, many of these
questions now have satisfactory answers. Without exhausting this still
growing line of remarkable work, we would like to refer to the papers
\cite{mossel2012stochastic,massoulie2014ramanujan,abbe2015multiple,mossel2013proof,mossel2015reconstruction}
for weak recovery, \cite{mossel2016bisection,abbe2015general,amini2014semidefinite,gao2015achieving,yun2014accurate}
for almost exact recovery, and \cite{mossel2016bisection,abbe2016exact,abbe2015general}
for exact recovery. SDP has in fact been shown to achieve the optimal
exact recovery threshold \cite{hajek2016achieving,perry2015semidefinite,bandeira2015convex,agarwal2015multisection}.
Our results imply sufficient conditions for SDP achieving these various
types of recovery, and moreover interpolate between them. As mentioned,
we are mostly concerned with the non-asymptotic setting with a growing
number of clusters, without attempting to optimize the values of the
leading constants. Therefore, our results focus on somewhat different
regimes from the work above.

Particularly relevant to us is the work in \cite{ChinRaoVu15,yun2014accurate,yun2016optimal,abbe2015general,gao2015achieving,Guedon2015,makarychev2016learning},
which provides explicit bounds on the error rates of other algorithms
for estimating the ground-truth clustering in SBM. The Censored Block
Model is studied in the papers \cite{abbe2014censored,hajek2015censor,heimlicher2012label,ChinRaoVu15,lelarge2015reconstruction,saade2015spectral}.
Robustness issues in SBM are considered in the work in \cite{Cai2014robust,frieze1997algoRandomGraphs,feige2001semirandom,krivelevich2006coloring,perry2015semidefinite,ricci2016performance,coja2004coloringSemirandom,moitra2016robust,makarychev2016learning}.
We discuss these results in more details in Section~\ref{sec:main}.

\subsection{Notations}

Column vectors are denoted by lower-case bold letters such as $\mathbf{u}$,
where $u_{i}$ is its $i$-th entry. Matrices are denoted by bold
capital letters such as $\M$, with $\M^{\top}$ denoting the transpose
of $\M$, $\Tr(\M)$ its trace, $M_{ij}$ its $(i,j)$-th entry, and
$\mbox{diag}\left(\M\right)$ the vector of its diagonal entries.
For a matrix $\M$, $\norm[\M]1:=\sum_{i,j}M_{ij}$ is its entry-wise
$\ell_{1}$ norm, $\norm[\M]{\infty}:=\max_{i,j}\left|M_{ij}\right|$
the entry-wise $\ell_{\infty}$ norm, and $\opnorm{\M}$ the spectral
norm (the maximum singular value). Denote by $\M_{i\bullet}$ the
$i$-th row of the matrix $\M$ and $\M_{\bullet j}$ its $j$-th
column. We write $\M\succeq0$ if $\M$ is symmetric and positive
semidefinite. With another matrix $\G$ of the same dimension as $\M$,
we let $\left\langle \M,\G\right\rangle :=\Tr(\M^{\top}\G)$ denote
their trace inner product, and use $\M\geq\G$ to mean that $M_{ij}\geq G_{ij}$
for all $i,j\in[\num]$. Let $\I$ and $\OneMat$ be the $\num\times\num$
identity matrix and all-one matrix, respectively, and $\onevec$ the
all-one column vector of length $\num$. 

We use $\Bern(\mu)$ to denote the Bernoulli distribution with rate
$\mu\in[0,1]$. For a positive integer~$i$, let $[i]\coloneqq\{1,2,\ldots,i\}$.
For a real number $x$, $\left\lceil x\right\rceil $ denotes its
ceiling. Throughout a universal constant~$C$ means a fixed number
that is independent of the model parameters ($\num$, $\numclust$,
$\inprob$, $\outprob$, etc.) and the graph distribution. We use
the following standard notations for order comparison of two non-negative
sequences $\{a_{\num}\}$ and $\{b_{\num}\}$: We write $a_{\num}=O(b_{\num})$,
$b_{\num}=\Omega(a_{\num})$ or $a_{\num}\lesssim b_{\num}$ if there
exists a universal constant $C>0$ such that $a_{\num}\le Cb_{\num}$
for all $\num$. We write $a_{\num}=\Theta(b_{\num})$ or $a_{\num}\asymp b_{\num}$
if $a_{n}=O(b_{\num})$ and $a_{\num}=\Omega(b_{\num})$. We write
$a_{\num}=o(b_{\num})$ or $b_{\num}=\omega(a_{\num})$ if $\lim_{n\to\infty}a_{\num}/b_{\num}=0$.

\section{Problem setup\label{sec:setup}}

In this section, we formally set up the problem of cluster estimation
under SBM and describe the SDP approach.

\subsection{The Stochastic Block Model\label{sec:setup_model}}

Given $\num$ nodes, we assume that each node belongs to exactly one
of $\numclust$ ground truth clusters, where the clusters have equal
size. This ground truth is encoded in the cluster matrix $\Ystar\in\{0,1\}^{\num\times\num}$
as defined in Section~\ref{sec:intro}. We do not know $\Ystar$,
but we observe the adjacency matrix $\Adj$ of a graph generated from
the following Stochastic Block Model (SBM).

\begin{mdl}[Standard Stochastic Block Model]\label{mdl:SBM} The
graph adjacency matrix $\Adj\in\{0,1\}^{\num\times\num}$ is symmetric
with its entries $\{\adj_{ij},i<j\}$ generated independently by 
\[
\adj_{ij}\sim\begin{cases}
\Bern(\inprob) & \text{if }\ystar_{ij}=1,\\
\Bern(\outprob) & \text{if }\ystar_{ij}=0,
\end{cases}
\]
where $0\le\outprob<\inprob\le1.$ \end{mdl}

\noindent The values of the diagonal entries of $\Adj$ are inconsequential
for the SDP formulation~(\ref{eq:SDP1}) due to the constraint $Y_{ii}=1,\forall i$.
Therefore, we assume without loss of generality that $A_{ii}\sim\Bern(\inprob)$
independently for all $i\in[\num]$, which simplifies the presentation
of the analysis

Our goal is to estimate $\Ystar$ given the observed graph $\Adj$.
Let $\LabelStar\in[\numclust]^{\num}$ be the vector of ground-truth
cluster labels, where $\labelstar_{i}$ is the index of the cluster
that contains node $i$. (The cluster labels are unique only up to
permutation; here $\LabelStar$ is defined with respect to an arbitrary
permutation.) Playing a crucial role in our results is the quantity
\begin{equation}
\snr\coloneqq\frac{(\inprob-\outprob)^{2}}{\frac{1}{\numclust}\inprob+(1-\frac{1}{\numclust})\outprob},\label{eq:snr}
\end{equation}
which is a measure of the SNR of the model. In particular, the numerator
of $\snr$ is the squared expected difference between the in- and
cross-cluster edge probabilities, and the denominator is essentially
the average variance of the entries of $\A$. The quantity $\snr$
has been shown to capture the hardness of SBM, and defines the celebrated
Kesten-Stigum threshold~\cite{mossel2015reconstruction}. To avoid
cluttered notation, we assume throughout the paper that $n\ge4$,
$2\le\numclust<\num$ and there exists a universal constant $0<c<1$
that $q\ge cp$; this setting encompasses most interesting regimes
of the problem, as clustering is more challenging when $\outprob$
is large.

\subsection{Semidefinite programming relaxation\label{sec:setup_sdp}}

We consider the SDP formulation in (\ref{eq:SDP1}), whose optimal
solution $\Yhat$ serves as an estimator of ground-truth cluster matrix
$\Ystar$. This SDP can be interpreted as a convex relaxation of the
maximum likelihood estimator, the modularity maximization problem,
the optimal subgraph/cut problem, or a variant of the robust/sparse
PCA problem; see \cite{amini2014semidefinite,bollobas2004maxcut,chen2012clustering,ChenXu2016,CMM}
for such derivations. Our goal is to study the recovery error of $\Yhat$
in terms of the number of nodes $\num$, the number of clusters $\numclust$
and the SNR measure~$\snr$ defined above. 

Note that there is nothing special about the particular formulation
in (\ref{eq:SDP1}). All our results apply to, for example, the alternative
SDP formulation below:

\begin{equation}
\begin{aligned}\Yhat=\argmax_{\mathbf{Y}\in\mathbb{R}^{\num\times\num}}\; & \left\langle \Y,\Adj\right\rangle \\
\mbox{s.t.}\; & \Y\succeq0,\;0\le\Y\le\J,\\
 & \sum_{i,j=1}^{\num}Y_{ij}=\sum_{i,j=1}^{\num}\ystar_{ij},\\
 & Y_{ii}=1,\forall i\in[\num].
\end{aligned}
\label{eq:SDP2}
\end{equation}
This formulation was previously considered in \cite{ChenXu2016}.
We may also replace the third constraint above with the row-wise constraints
$\sum_{j=1}^{\num}Y_{ij}=\sum_{j=1}^{\num}\ystar_{ij},\forall i\in[\num]$,
akin to the formulation in \cite{amini2014semidefinite} motivated
by weak assortative SBM. Under the standard assumption of equal-sized
clusters, the values $\sum_{i,j=1}^{\num}\ystar_{ij}=\num^{2}/\numclust$
and $\sum_{j=1}^{\num}\ystar_{ij}=\num/\numclust$ are known. Therefore,
the formulation~(\ref{eq:SDP2}) has the advantage that it does not
require knowledge of the edge probabilities $\inprob$ and $\outprob$,
but instead the number of clusters~$\numclust$.\footnote{Note that the constraint $\Y\le\OneMat$ in the formulations~(\ref{eq:SDP1})
and~(\ref{eq:SDP2}) is in fact redundant, as it is implied by the
constraints $\Y\succeq0$ and $Y_{ii}=1,\forall i$. We still keep
this constraint as the property $\Yhat\le\OneMat$ plays a crucial
role in our analysis.}

The optimization problems in (\ref{eq:SDP1}) and (\ref{eq:SDP2})
can be solved in polynomial time using any general-purpose SDP solvers
or first order algorithms. Moreover, this SDP approach continues to
motivate, and benefit from, the rapid development of efficient algorithms
for solving structured SDPs. For example, the algorithms considered
in \cite{javanmard2016phase,ricci2016performance} can solve a problem
involving $\num=10^{5}$ nodes within seconds on a laptop. In addition
to computational efficiency, the SDP approach also enjoys several
other desired properties including robustness, conceptual simplicity
and applicability to sparse graphs, making it an attractive option
among other clustering and community detection algorithms. The empirical
performance of SDP has been extensively studied, both under SBM and
with real data; see for example the work in \cite{amini2014semidefinite,CMM,chen2012sparseclustering,javanmard2016phase,ricci2016performance}.
Here we focus on the theoretical guarantees of this SDP approach.

\subsection{Explicit clustering by $k$-medians\label{sec:setup_kmedian}}

After solving the SDP formulations (\ref{eq:SDP1}) or (\ref{eq:SDP2}),
the cluster membership can be extracted from the solution $\Yhat$.
This can be done using many simple procedures. For example, when $\norm[\Yhat-\Ystar]1<\frac{1}{2}$,
simply rounding the entries of $\Yhat$ will exactly recover $\Ystar$,
from which the true clusters can be extracted easily. In the case
with $\numclust=2$ clusters, one may use the signs of the entries
of first eigenvector of $\Yhat-\frac{1}{2}\OneMat$, a procedure analyzed
in \cite{Guedon2015,MontanariSen16}. More generally, our theoretical
results guarantee that the SDP solution $\Yhat$ is already close
to true cluster matrix $\Ystar$; in this case, we expect that many
local rounding/refinement procedures, such as Lloyd's-style greedy
algorithms~\cite{lu2016lloyd}, will be able to extract a high-quality
clustering.

For the sake of retaining focus on the SDP formulation, we choose
not to separately analyze these possible extraction procedures, but
instead consider a more unified approach. In particular, we view the
rows of $\Yhat$ as $\num$ points in $\real^{\num}$, and apply $k$-medians
clustering to them to extract the clusters. While exactly solving
the $k$-medians problem is computationally hard, there exist polynomial-time
constant-factor approximation schemes, such as the $6\frac{2}{3}$-approximation
algorithm in~\cite{charikar1999constant}, which suffices for our
purpose. Note that this algorithm may not be the most efficient way
to extract an explicit clustering from $\Yhat$; rather, it is intended
as a simple venue for deriving a clustering error bound that can be
readily compared with existing results.

Formally, we use $\apkmedian(\Yhat)$ to denote a $\apxconst$-approximate
$k$-median procedure applied to the rows of $\Yhat$; the details
are provided in Section~\ref{sec:kmedians}. The output 
\[
\LabelHat:=\apkmedian(\Yhat)
\]
is a vector in $[\numclust]^{\num}$ such that node $i$ is assigned
to the $\labelhat_{i}$-th cluster by the procedure. We are interested
in bounding the clustering error of $\LabelHat$ relative to the ground
truth $\LabelStar$. Let $S_{\numclust}$ denote the symmetric group
consisting of all permutations of $[\numclust]$; we consider the
metric
\begin{equation}
\misrate(\LabelHat,\LabelStar):=\min_{\pi\in S_{\numclust}}\frac{1}{n}\left|\left\{ i\in[\num]:\labelhat_{i}\neq\pi(\labelstar_{i})\right\} \right|,\label{eq:mis_rate_def}
\end{equation}
which is the proportion of nodes that are mis-classified, modulo permutations
of the cluster labels.\\

Before proceeding, we briefly mention several possible extensions
of the setting discussed above. The number $\frac{\inprob+\outprob}{2}$
in the SDP~(\ref{eq:SDP1}) can be replaced by a tuning parameter
$\lambda$; as would become evident from the proof, our theoretical
results in fact hold for an entire range of $\lambda$ values, for
example $\lambda\in[\frac{1}{4}\inprob+\frac{3}{4}\outprob,\frac{3}{4}\inprob+\frac{1}{4}\outprob]$.
Our theory also generalizes to the setting with unequal cluster sizes;
in this case the same theoretical guarantees hold with $\numclust$
replaced by $\num/\underbar{\ensuremath{\size}}$, where $\underbar{\ensuremath{\size}}$
is any lower bound of the cluster sizes.

\section{Main results\label{sec:main}}

We present in Section~\ref{sec:main_sbm} our main theorems, which
provide exponentially-decaying error bounds for the SDP formulation
under SBM. We also discuss the consequences of our results, including
their implications for robustness in Section~\ref{sec:main_robust}
and applications to the Censored Block Model in Section~\ref{sec:main_censor}.
In the sequel, $\Yhat$ denotes any optimal solution to the SDP formulation
in either~(\ref{eq:SDP1}) or~(\ref{eq:SDP2}). 

\subsection{Error rates under standard SBM\label{sec:main_sbm}}

In this section, we consider the standard SBM setting in Model~\ref{mdl:SBM}.
Recall that $\num$ and $\numclust$ are respectively the numbers
of nodes and clusters, and $\Ystar$ is the ground-truth cluster matrix
defined in Section~\ref{sec:intro} with $\LabelStar$ being the
corresponding vector of true cluster labels. Our results are stated
in terms of the SNR measure $\snr$ given in equation~(\ref{eq:snr}).

The first theorem, proved in Section~\ref{sec:proof_exp_rate}, shows
that the SDP solution $\Yhat$ achieves an exponential error rate.
\begin{thm}[Exponential Error Rate]
\emph{\label{thm:exp_rate} }Under Model~\ref{mdl:SBM}, there exist
universal constants $\consts,\constgamma,\conste>0$ for which the
following holds. If $\snr\geq\consts\numclust^{2}/\num$, then we
have
\[
\frac{\norm[\Yhat-\Ystar]1}{\norm[\Ystar]1}\leq\constgamma\exp\left[-\frac{\snr\num}{\conste\numclust}\right]
\]
with probability at least $1-\frac{7}{\num}-7e^{-\Omega(\sqrt{\num})}$.
\end{thm}
Our next result concerns the explicit clustering $\LabelHat:=\apkmedian(\Yhat)$
extracting from $\Yhat$, using the approximate $k$-medians procedure
given in Section~\ref{sec:setup_kmedian}, where $\apxconst=6\frac{2}{3}$.
As we show in the proof of the following theorem, the error rate in
$\LabelHat$ is always upper-bounded by the error in $\Yhat$:
\[
\misrate(\LabelHat,\LabelStar)\le\frac{86}{3}\cdot\frac{\norm[\Yhat-\Ystar]1}{\norm[\Ystar]1};
\]
cf.~Proposition~\ref{prop:kmedian}. Consequently, the number of
misclassified nodes also exhibits an exponential decay. 
\begin{thm}[Clustering Error]
\label{thm:cluster_error_rate}Under Model~\ref{mdl:SBM}, there
exist universal constants $\consts,\constmis,\conste>0$ for which
the following holds. If $\snr\geq\consts\numclust^{2}/\num$, then
we have
\[
\misrate(\LabelHat,\LabelStar)\leq\constmis\exp\left[-\frac{\snr\num}{\conste\numclust}\right]
\]
with probability at least $1-\frac{7}{\num}-7e^{-\Omega(\sqrt{\num})}$.
\end{thm}
We prove this theorem in Section \ref{sec:proof_cluster_error_rate}.
\\

Theorems~\ref{thm:exp_rate} and~\ref{thm:cluster_error_rate} are
applicable in the sparse graph regime with bounded expected degrees.
For example, suppose that $\numclust=2$, $\inprob=\frac{a}{\num}$
and $\outprob=\frac{b}{\num}$ for two constants $a,b$; the results
above guarantee a non-trivial accuracy for SDP (i.e., $\norm[\Yhat-\Ystar]1<\frac{1}{2}\norm[\Ystar]1$
or $\misrate(\LabelHat,\LabelStar)<\frac{1}{2}$) as long as $(a-b)^{2}\ge C(a+b)$
for some constant $C$. Another interesting regime that our results
apply to, is when there is a large number of clusters. For example,
for any constant $\epsilon\in(0,\frac{1}{2})$, if $\numclust=\num^{1/2-\epsilon}$
and $\inprob=2\outprob$, then SDP achieves exact recovery ($\misrate(\LabelHat,\LabelStar)<\frac{1}{\num}$)
provided that $\inprob\gtrsim\num^{-2\epsilon}$. 

Below we provide additional discussion of our results, and compare
with existing ones.

\subsubsection{Consequences and Optimality\label{sec:consequence}}

Theorems~\ref{thm:exp_rate} and~\ref{thm:cluster_error_rate} immediately
imply sufficient conditions for the various recovery types discussed
in Section~\ref{sec:related}. 
\begin{itemize}
\item \textbf{Exact recovery (strong consistency):} When $\snr\gtrsim\frac{\numclust^{2}+\numclust\log\num}{n}$,
Theorem~\ref{thm:exp_rate} guarantees that $\norm[\Yhat-\Ystar]1<\frac{1}{2}$
with high probability, in which case element-wise rounding $\Yhat$
exactly recovers the true cluster matrix $\Ystar$.\footnote{In fact, a simple modification of our analysis proves that $\Yhat=\Ystar$.
We omit the details of such refinement for the sake of a more streamlined
presentation.} This result matches the best known exact recovery guarantees for
SDP (and other polynomial-time algorithms) when $\numclust$ is allowed
to grow with $\num$; see \cite{ChenXu2016,amini2014semidefinite}
for a review of theses results.
\item \textbf{Almost exact recovery (weak consistency):} Under the condition
$\snr=\omega(\frac{\numclust^{2}}{\num})$, Theorem~\ref{thm:cluster_error_rate}
ensures that $\misrate(\LabelHat,\LabelStar)=o(1)$ with high probability
as $\num\to\infty$, hence SDP achieves weak consistency. This condition
is optimal (necessary and sufficient), as has been proved in~\cite{mossel2016bisection,abbe2015general}.
\item \textbf{Weak recovery:} When $\snr\gtrsim\frac{\numclust^{2}}{\num}$,
Theorem~\ref{thm:cluster_error_rate} ensures that $\misrate(\LabelHat,\LabelStar)<1-\frac{1}{\numclust}$
with high probability, hence SDP achieves weak recovery. In particular,
in the setting with $\numclust=2$ clusters, SDP recovers a clustering
that is positively correlated with the ground-truth under the condition
$\snr\gtrsim\frac{1}{\num}$. This condition matches up to constants
the so-called Kesten-Stigum (KS) threshold $\snr>\frac{1}{\num}$,
which is known to be optimal \cite{mossel2012stochastic,massoulie2014ramanujan,abbe2015multiple,mossel2013proof,mossel2015reconstruction}. 
\item \textbf{Recovery with $\delta$ error:} More generally, for any number
$\delta\in(0,1)$, Theorem~\ref{thm:cluster_error_rate} implies
that if $\snr\gtrsim\max\left\{ \frac{\numclust^{2}}{\num},\frac{\numclust}{\num}\log\frac{1}{\delta}\right\} $,
then $\misrate(\LabelHat,\LabelStar)<\delta$ with high probability.
In the case with $\numclust=2$, the minimax rate result in \cite{gao2015achieving}
implies that $\snr\gtrsim\frac{1}{\num}\log\frac{1}{\delta}$ is necessary\emph{
}for \emph{any} algorithm to achieve a $\delta$ clustering error.
Our results are thus optimal up to a multiplicative constant. 
\end{itemize}
Our results therefore cover these different recovery regimes by a
unified error bound, using a single algorithm. This can be contrasted
with the previous error bound~(\ref{eq:GV_error_bound}) proved using
the Grothendieck's inequality approach, which fails to identify the
exact recovery condition above. In particular, the bound~(\ref{eq:GV_error_bound})
decays polynomially with the SNR measure $\snr$; since $\snr$ is
at most $\numclust$ and $\norm[\Ystar]1=\num^{2}/\numclust$, the
smallest possible error that can be derived from this bound is $\norm[\Yhat-\Ystar]1=O(\sqrt{\num^{3}/\numclust})$.

Our results apply to general values of $\numclust$, which is allowed
to scale with $\num$, hence the size of the clusters can be sublinear
in $\num$. We note that in this regime, a computational-barrier phenomenon
seems to take place: there may exist instances of SBM in which cluster
recovery is information-theoretically possible but cannot be achieved
by computationally efficient algorithms. For example, the work in~\cite{ChenXu2016}
proves that the intractable maximum likelihood estimator succeeds
in exact recovery when $\snr\gtrsim\frac{\numclust\log n}{n}$; it
also provides evidences suggesting that all efficient algorithms fail
unless $\snr\gtrsim\frac{\numclust^{2}+\numclust\log\num}{n}$. Note
that the latter is consistent with the condition derived above from
our theorems. 

The above discussion has the following implications for the optimality
of Theorems~\ref{thm:exp_rate} and~\ref{thm:cluster_error_rate}.
On the one hand, the general minimax rate result in~\cite{gao2015achieving}
suggests that all algorithms (regardless of their computational complexity)
incur at least $\exp\left[-\Theta(\snr\num/\numclust)\right]$ error.
Our exponential error rate matches this information-theoretic lower
bound. On the other hand, in view of the computational barrier discussed
in the last paragraph, our SNR condition $\snr\gtrsim\numclust^{2}/\num$
is likely to be unimprovable if efficient algorithms are considered.

\subsubsection{Comparison with existing results}

We discuss some prior work that also provides efficient algorithms
attaining an exponentially-decaying rate for the clustering error
$\misrate(\LabelHat,\LabelStar)$. To be clear, these algorithms are
very different from ours, often involving a two-step procedure that
first computes an accurate initial estimate (typically by spectral
clustering) followed by a ``clean-up'' process to obtain the final
solution. Some of them require additional steps of sample splitting
and graph trimming/regularization.  As we discussed in Section~\ref{sec:main_robust}
below, many of these procedures rely on delicate properties of the
standard SBM, and therefore are not robust against model deviation.

Most relevant to us is the work in \cite{ChinRaoVu15}, which develops
a spectral algorithm with sample splitting. As stated in their main
theorem, their algorithm achieves the error rate $\exp\left[-\Omega(\snr\num/\numclust^{2})\right]$
when $\snr\gtrsim\numclust^{2}/\num$, as long as $\numclust$ is
a fixed constant when $\num\to\infty$. The work in \cite{yun2014accurate}
and \cite{yun2016optimal} also considers spectral algorithms, which
attain exponential error rates assuming that $\numclust$ is a constant
and $\inprob n\to\infty$. The algorithms in \cite{gao2015achieving,gao2016DCBM}
involves obtaining an initial clustering using spectral algorithms,
which require $\snr\gg\numclust^{3}/n$; a post-processing step (e.g.,
using a Lloyd's-style algorithm \cite{lu2016lloyd}) then outputs
a final solution that asymptotically achieves the minimax error rate
$\exp\left[-I^{*}\cdot n/\numclust\right],$ where $I^{*}$ is an
appropriate form of Renyi divergence and satisfies $I^{*}\asymp\snr$.
The work in \cite{abbe2015general} proposes an efficient algorithm
called Sphere Comparison, which achieves an exponential error rate
in the constant degree regime $\inprob=\Theta(1/n)$ when $\snr\ge\numclust^{2}/n.$
The work \cite{makarychev2016learning} uses SDP to produce an initial
clustering solution to be fed to another clustering algorithm; their
analysis extends the techniques in \cite{Guedon2015} to the setting
with corrupted observations, and their overall algorithm attains an
exponential error rate assuming that $\snr\gtrsim\numclust^{4}/n$.

\subsection{Robustness\label{sec:main_robust}}

Compared to other clustering algorithms, one notable advantage of
the SDP approach lies in its robustness under various challenging
settings of SBM. For example, standard spectral clustering is known
to be inconsistent in the sparse graph regime with $\inprob,\outprob=O(1/\num)$
due to the existence of atypical node degrees, and alleviating this
difficulty generally requires sophisticated algorithmic techniques.
In contrast, as shown in Theorem~\ref{thm:exp_rate} as well as other
recent work~\cite{MontanariSen16,Guedon2015,CMM}, the SDP approach
is applicable without change to this sparse regime. SDP is also robust
against the existence of $o(\num)$ outlier nodes and/or edge modifications,
while standard spectral clustering is fairly fragile in these settings~\cite{Cai2014robust,perry2015semidefinite,moitra2016robust,ricci2016performance,MontanariSen16,makarychev2016learning}.

Here we focus on another remarkable form of robustness enjoyed by
SDP with respect to heterogeneous edge probabilities and monotone
attack, captured in the following generalization of the standard SBM.

\begin{mdl}[Heterogeneous Stochastic Block Model]\label{mdl:heterSBM}Given
the ground-truth clustering $\LabelStar$ (encoded in the cluster
matrix $\Ystar$), the entries $\{\adj_{ij},i<j\}$ of the graph adjacency
matrix $\A$ are generated independently with 
\[
\begin{cases}
\text{\ensuremath{\adj_{ij}} is Bernoulli with rate \emph{at least } \ensuremath{\inprob}} & \text{if }\ystar_{ij}=1,\\
\text{\ensuremath{\adj_{ij}} is Bernoulli with rate \emph{at most } \ensuremath{\outprob}} & \text{if }\ystar_{ij}=0,
\end{cases}
\]
where $0\le\outprob<\inprob\le1$.\end{mdl}

The above model imposes no constraint on the edge probabilities besides
the upper/lower bounds, and in particular the probabilities can be
non-uniform. This model encompasses a variant of the so-called \emph{monotone
attack} studied extensively in the computer science literature \cite{feige2001semirandom,krivelevich2006coloring,coja2004coloringSemirandom}:
here an adversary can arbitrarily set some edge probabilities to $1$
or $0$, which is equivalent to adding edges to nodes in the same
cluster and removing edges across clusters.\footnote{We do note that here the addition/removal of edges are determined
before the realization of the random edge connections, which is more
restrictive than the standard monotone attack model. We believe this
restriction is an artifact of the analysis, and leave further improvements
to future work.} Note that the adversary can make far more than $o(n)$ edge modifications
\textemdash{} $O(\num^{2})$ to be precise \textemdash{} in a restrictive
way that seems to strengthen the clustering structure (hence the name).
Monotone attack however does not necessarily make the clustering problem
easier. On the contrary, the adversary can significantly alter some
predictable structures that arise in standard SBM (such as the graph
spectrum, node degrees, subgraph counts and the non-existence of dense
spots~\cite{feige2001semirandom}), and hence foil algorithms that
over-exploit such structures. For example, some spectral algorithms
provably fail in this setting~\cite{coja2004coloringSemirandom,ricci2016performance}.
More generally, Model~\ref{mdl:heterSBM} allows for unpredictable,
non-random deviations (not necessarily due to an adversary) from the
standard SBM setting, which has statistical properties that are rarely
possessed by real world graphs.

It is straightforward to show that when exact recovery is concerned,
SDP is unaffected by the heterogeneity in Model~\ref{mdl:heterSBM};
see~\cite{feige2001semirandom,hajek2016achieving,chen2012sparseclustering}.
The following theorem, proved in Section~\ref{sec:proof_exp_rate_robust},
shows that SDP in fact achieves the same exponential error rates in
the presence of heterogeneity.
\begin{thm}[Robustness]
\emph{\label{thm:exp_rate_robust} }The conclusions in Theorems~\ref{thm:exp_rate}
and~\ref{thm:cluster_error_rate} continue to hold under Model~\ref{mdl:heterSBM}.
\end{thm}
Consequently, under the same conditions discussed in Section~\ref{sec:consequence},
the SDP approach achieves exact recovery, almost exact recovery, weak
recovery and a $\delta$-error in the more general Model~\ref{mdl:heterSBM}\textbf{.}
\\

As a passing note, the results in \cite{moitra2016robust} show that
when exact constant values are concerned, the optimal weak recovery
threshold changes in the presence of monotone attack, and there may
exist a fundamental tradeoff between optimal recovery in standard
SBM and robustness against model deviation.

\subsection{Censored block model\label{sec:main_censor}}

The Censored Block Model~\cite{abbe2014censored} is a variant of
the standard SBM that represents the scenario with partially observed
data, akin to the settings of matrix completion~\cite{keshavan2009matrix}
and graph clustering with measurement budgets~\cite{ChenXu2016}.
In this section, we show that Theorems~\ref{thm:exp_rate} and~\ref{thm:cluster_error_rate}
immediately yield recovery guarantees for the SDP formulation under
this model. 

Concretely, again assume a ground-truth set of $\numclust$ equal-size
clusters over $\num$ nodes, with the corresponding label vector $\LabelStar\in[\numclust]^{\num}$
. These clusters can be encoded by the cluster matrix $\Ystar\in\{0,1\}^{\num\times\num}$
as defined in Section~\ref{sec:intro}, but it is more convenient
to work with its $\pm1$ version $2\Ystar-\OneMat$. Under the Censored
Block Model, one observes the entries of $2\Ystar-\OneMat$ restricted
to the edges of an Erdos-Renyi graph $G(\num,\obsprob)$, but with
each entry flipped with probability $\flipprob<\frac{1}{2}$. The
model is described formally below.

\begin{mdl}[Censored Block Model]\label{mdl:censor} The observed
matrix $\CensorMat\in\{-1,0,1\}^{\num\times\num}$ is symmetric and
has entries generated independently across all $i<j$ with 
\[
\censormat_{ij}=\begin{cases}
0, & \text{with probability \ensuremath{1-\obsprob},}\\
2\ystar_{ij}-1 & \text{with probability \ensuremath{\obsprob(1-\flipprob)},}\\
-(2\ystar_{ij}-1) & \text{with probability \ensuremath{\obsprob\flipprob}},
\end{cases}
\]
where $0<\obsprob\le1$ and $0<\flipprob<\frac{1}{2}$.\end{mdl}

\noindent The goal is again to recovery $\Ystar$ (equivalently $2\Ystar-\OneMat$),
given the observed matrix $\CensorMat$. 

One may reduce this problem to the standard SBM by constructing an
adjacency matrix $\Adj\in\{0,1\}^{\num\times\num}$ with $\adj_{ij}=|\censormat_{ij}|\cdot(\censormat_{ij}+1)/2$;
that is, we zero out the unobserved entries in the binary representation
$(\CensorMat+\OneMat)/2$ of $\Z$. The upper-triangular entries of
$\Adj$ are independent Bernoulli variables with 
\[
\P\left(\adj_{ij}=1\right)=\P\left(\censormat_{ij}=1\right)=\begin{cases}
\obsprob(1-\flipprob) & \text{if \ensuremath{\ystar_{ij}=1}},\\
\ensuremath{\obsprob\flipprob} & \text{if \ensuremath{\ystar_{ij}=0}}.
\end{cases}
\]
Therefore, the matrix $\Adj$ can be viewed as generated from the
standard SBM (Model~\ref{mdl:SBM}) with $\inprob=\obsprob(1-\flipprob)$
and $\outprob=\obsprob\flipprob$. We can then obtain an estimate
$\Yhat$ of $\Ystar$ by solving the SDP formulation~(\ref{eq:SDP1})
or~(\ref{eq:SDP2}) with $\Adj$ as the input, possibly followed
by the approximate $k$-medians procedure to get an explicit clustering
$\LabelHat$. The error rates of $\Yhat$ and $\LabelHat$ can be
derived as a corollary of Theorems~\ref{thm:exp_rate} and~\ref{thm:cluster_error_rate}.
\begin{cor}[Censored Block Model]
\label{cor:censor}Under Model~\ref{mdl:censor}, there exist universal
constants $\consts,\constgamma,\conste>0$ for which the following
holds. If $\obsprob(1-2\flipprob)^{2}\ge\consts\frac{\numclust^{2}}{\num},$
then
\[
\frac{3}{86}\misrate(\LabelHat,\LabelStar)\le\frac{\norm[\Yhat-\Ystar]1}{\norm[\Ystar]1}\leq\constgamma\exp\left[-\obsprob(1-2\flipprob)^{2}\cdot\frac{\num}{\conste\numclust}\right]
\]
with probability at least $1-\frac{7}{\num}-7e^{-\Omega(\sqrt{\num})}$.
\end{cor}
Specializing this corollary to the different types of recovery defined
in Section~\ref{sec:related}, we immediately obtain the following
sufficient conditions for SDP under the Censored Block Model: (a)
exact recovery is achieved when $\obsprob\gtrsim\frac{\max\left\{ \numclust\log\num,\numclust^{2}\right\} }{\num(1-2\flipprob)^{2}}$;
(b) almost exact recovery is achieved when $\obsprob=\omega\left(\frac{\numclust^{2}}{\num(1-2\flipprob)^{2}}\right)$,
(c) weak recovery is achieved when $\obsprob\gtrsim\frac{\numclust^{2}}{\num(1-2\flipprob)^{2}}$;
(d) a $\delta$ clustering error is achieved when $\obsprob\gtrsim\frac{\max\left\{ \numclust^{2},\numclust\log(1/\delta)\right\} }{\num(1-2\flipprob)^{2}}$.

Several existing results focus on the Censored Block Model with $\numclust=2$
clusters in the asymptotic regime $\num\to\infty$. In this setting,
the work in \cite{abbe2014censored} proves that exact recovery is
possible if and only if $\obsprob>\frac{2\log\num}{\num(1-2\flipprob)^{2}}$
in the limit $\epsilon\to1/2$, and provides an SDP-based algorithm
that succeeds twice the above threshold; a more precise threshold
$\obsprob>\frac{\log\num}{\num\left(\sqrt{1-\epsilon}-\sqrt{\epsilon}\right)^{2}}$
is given in \cite{hajek2015censor}. For weak recovery under a sparse
graph $\obsprob=\Theta(1/\num)$, it is conjectured in \cite{heimlicher2012label}
that the problem is solvable if and only if $\obsprob>\frac{1}{\num\left(1-2\flipprob\right)^{2}}$.
The converse and achievability parts of the conjecture are proved
in \cite{lelarge2015reconstruction} and \cite{saade2015spectral},
respectively. Corollary~\ref{cor:censor} shows that SDP achieves
(up to constants) the above exact and weak recovery thresholds; moreover,
our results apply to the more general setting with $\numclust\ge2$
clusters. 

\section{Proof of Theorem \ref{thm:exp_rate}\label{sec:proof_exp_rate}}

In this section we prove our main theoretical results in Theorem~\ref{thm:exp_rate}
for Model~\ref{mdl:SBM} (Standard SBM). While Model~\ref{mdl:SBM}
is a special case of Model~\ref{mdl:heterSBM} (Heterogeneous SBM),
we choose not to deduce Theorem~\ref{thm:exp_rate} as a corollary
of Theorem~\ref{thm:exp_rate_robust} which concerns the more general
model. Instead, to highlight the main ideas of the analysis and avoid
technicalities, we provide a separate proof of Theorem~\ref{thm:exp_rate}.
In Section~\ref{sec:proof_exp_rate_robust} to follow, we show how
to adapt the proof to Model~\ref{mdl:heterSBM}.\\

Before going into the details, we make a few observations that simplify
the proof. First note that it suffices to prove the theorem for the
first SDP formulation~(\ref{eq:SDP1}). Indeed, the ground-truth
matrix $\Ystar$ is also feasible to second formulation~(\ref{eq:SDP2});
moreover, thanks to the equality constraint $\sum_{i,j=1}^{\num}Y_{ij}=\sum_{i,j=1}^{\num}\ystar_{ij}$,
subtracting the constant-valued term $\left\langle \Y,\frac{\inprob+\outprob}{2}\OneMat\right\rangle $
from the objective of~(\ref{eq:SDP2}) does not affect its optimal
solutions. The two formulations are therefore identical except for
the above equality constraint, which is never used in the proof below.
Secondly, under the assumption $c\inprob\le\outprob\le\inprob$ for
a universal constant $c>0$, we have $\frac{1}{\numclust}\inprob+(1-\frac{1}{\numclust})\outprob\ge c\inprob$.
Therefore, it suffices to prove the theorem with the SNR redefined
as $\snr=\frac{(\inprob-\outprob)^{2}}{\inprob}$, which only affects
the universal constant $\conste$ in the exponent of the error bound.
Thirdly, it is in fact sufficient to prove the bound
\begin{equation}
\norm[\Yhat-\Ystar]1\leq\constgamma\num^{2}\exp\left[-\frac{\snr\num}{\conste\numclust}\right].\label{eq:equivalent_bound}
\end{equation}
Suppose that this bound holds; under the premise $\snr\geq\consts\numclust^{2}/\num$
of the theorem with the constant $\consts$ sufficiently large, we
have $\exp\left[-\frac{\snr\num}{2\conste\numclust}\right]\le\exp\left[-\frac{\consts}{2\conste}\cdot\numclust\right]\le\frac{1}{\numclust},$
hence the RHS of the bound (\ref{eq:equivalent_bound}) is at most
\[
\constgamma\num^{2}\cdot\frac{1}{\numclust}\cdot\exp\left[-\frac{\snr\num}{2\conste\numclust}\right]=\constgamma\norm[\Ystar]1\exp\left[-\frac{\snr\num}{2\conste\numclust}\right],
\]
which implies the error bound in theorem statement again up to a change
in the universal constant $\conste$. Finally, we define the convenient
shorthands $\error\coloneqq\norm[\Yhat-\Ystar]1$ for the error and
$\size:=\frac{\num}{\numclust}$ for the cluster size, which will
be used throughout the proof.\\

Our proof begins with a basic inequality using optimality. Since $\Ystar$
is feasible to the SDP (\ref{eq:SDP1}) and $\Yhat$ is optimal, we
have 
\begin{align}
0 & \le\left\langle \Yhat-\Ystar,\Adj-\frac{\inprob+\outprob}{2}\mathbf{J}\right\rangle =\left\langle \Yhat-\Ystar,\E\Adj-\frac{\inprob+\outprob}{2}\mathbf{J}\right\rangle +\left\langle \Yhat-\Ystar,\Adj-\E\Adj\right\rangle .\label{eq:basic_inequality}
\end{align}
A simple observation is that the entries of the matrix $\Ystar-\Yhat$
have matching signs with those of $\E\Adj-\frac{\inprob+\outprob}{2}\mathbf{J}$.
This observation implies the following relationship between the first
term on the RHS of equation~(\ref{eq:basic_inequality}) and the
error $\error\coloneqq\norm[\Yhat-\Ystar]1$.
\begin{fact}
\label{fact:expectation_gamma}We have the inequality
\begin{equation}
\left\langle \Ystar-\Yhat,\E\Adj-\frac{\inprob+\outprob}{2}\mathbf{J}\right\rangle \ge\frac{\inprob-\outprob}{2}\error.\label{eq:expectation_gamma}
\end{equation}
\end{fact}
The proof of this fact is deferred to Section~\ref{sec:proof_fact_expectation_gamma}.
Taking this fact as given and combining with the inequality (\ref{eq:basic_inequality}),
we obtain that 
\begin{equation}
\frac{\inprob-\outprob}{2}\error\le\left\langle \Yhat-\Ystar,\Adj-\E\Adj\right\rangle .\label{eq:gamma_bound}
\end{equation}

To bound the error $\error$, it suffices to control the RHS of equation~(\ref{eq:gamma_bound}),
where we depart from existing analysis. The seminal work in~\cite{Guedon2015}
bounds the RHS by a direct application of the Grothendieck's inequality.
As we discuss below, this argument fails to expose the fast, exponential
decay of the error $\error$. Our analysis develops a more precise
bound. To describe our approach, some additional notation is needed.
Let $\mathbf{U}\in\mathbb{R}^{\num\times\numclust}$ be the matrix
of the left singular vectors of $\Ystar$. Define the projection $\PT(\M)\coloneqq\mathbf{U}\mathbf{U}^{\top}\M+\M\mathbf{U}\mathbf{U}^{\top}-\mathbf{U}\mathbf{U}^{\top}\M\mathbf{U}\mathbf{U}^{\top}$
and its orthogonal complement $\PTperp(\M)=\M-\PT(\M)$ for any $\M\in\mathbb{R}^{\num\times\num}$.
Our crucial observation is that we should control $\left\langle \Yhat-\Ystar,\Adj-\E\Adj\right\rangle $
by separating the contributions from two projected components of $\Yhat-\Ystar$
defined by $\PT$ and $\PTperp$. In particular, we rewrite the inequality
(\ref{eq:gamma_bound}) as
\begin{align}
\frac{\inprob-\outprob}{2}\error & \leq\underbrace{\left\langle \PT(\Yhat-\Ystar),\Adj-\E\Adj\right\rangle }_{S_{1}}+\underbrace{\left\langle \PTperp(\Yhat-\Ystar),\Adj-\E\Adj\right\rangle }_{S_{2}}.\label{eq:gamma_S1_plus_S2}
\end{align}
 The first term $S_{1}$ involves the component of $\Yhat-\Ystar$
that is ``aligned'' with $\Ystar$; in particular, $\PT(\Yhat-\Ystar)$
is the orthogonal projection onto the subspace spanned by matrices
with the same column or row space as $\Ystar$. The second terms $S_{2}$
involves the orthogonal component $\PTperp(\Yhat-\Ystar)$, whose
column and row spaces are orthogonal to those of $\Ystar$. The main
steps of our analysis consist of bounding $S_{1}$ and $S_{2}$ separately.

The following proposition bounds the term $S_{1}$ and is proved in
Section \ref{sec:proof_prop_S1} to follow.
\begin{prop}
\label{prop:S1}Under the conditions of Theorem \ref{thm:exp_rate},
with probability at least $1-\frac{6}{\num\numclust}-2(\frac{e}{2})^{-2\num}$,
at least one of the following inequalities hold: 
\begin{align}
\error & \le\constgamma\num^{2}e^{-\snr\num/\left(2\conste\numclust\right)},\nonumber \\
S_{1} & \le D_{1}\error\sqrt{\frac{\inprob\log\left(\num^{2}/\error\right)}{\size}},\label{eq:S1bound}
\end{align}
where $D_{1}=12D=12\sqrt{7}$.
\end{prop}
Our next proposition, proved in Section \ref{sec:proof_prop_S2} to
follow, controls the term $S_{2}$.
\begin{prop}
\label{prop:S2}Under the conditions of Theorem \ref{thm:exp_rate},
with probability $1-\frac{1}{\num^{2}}-e^{-(3-\ln2)\num}-4e^{-c'\sqrt{\num}}$,
at least one of the following inequalities hold:
\begin{align}
\error & \le\constgamma\num^{2}e^{-\snr\num/(2\conste\numclust)},\nonumber \\
S_{2} & \le D_{2}\sqrt{\inprob\num}\frac{\error}{\size}+\frac{1}{8}(\inprob-\outprob)\error.\label{eq:S2bound}
\end{align}
where $D_{2}>0$ is a universal constant. 
\end{prop}
Equipped with these two propositions, the desired bound (\ref{eq:equivalent_bound})
follows easily. If the first inequality in the two propositions holds,
then we are done. Otherwise, there must hold the inequalities~(\ref{eq:S1bound})
and~(\ref{eq:S2bound}), which can be plugged into the RHS of equation~(\ref{eq:gamma_S1_plus_S2})
to get
\[
\frac{\inprob-\outprob}{2}\error\le D_{1}\error\sqrt{\frac{\inprob\log\left(\num^{2}/\error\right)}{\size}}+D_{2}\sqrt{\frac{\inprob\numclust^{2}}{\num}}\error+\frac{1}{8}(\inprob-\outprob)\error.
\]
Under the premise $\snr\geq\consts\numclust^{2}/\num$ of Theorem
\ref{thm:exp_rate}, we know that $D_{2}\sqrt{\frac{\inprob\numclust^{2}}{\num}}\le\frac{\inprob-\outprob}{8}$,
whence
\[
\frac{\inprob-\outprob}{4}\error\le D_{1}\error\sqrt{\frac{\inprob\log\left(\num^{2}/\error\right)}{\size}}.
\]
Doing some algebra yields the inequality $\error\le\num^{2}\exp\left[-\snr\num/(16D_{1}^{2})\right]$,
so the desired bound (\ref{eq:equivalent_bound}) again holds.\\

The rest of this section is devoted to establishing Propositions \ref{prop:S1}
and \ref{prop:S2}. Before proceeding to the proofs, we remark on
the above arguments and contrast them with alternative approaches.

\textbf{Comparison with the Grothendieck's inequality approach: }The
arguments in the work \cite{Guedon2015} also begin with a version
of the inequality~(\ref{eq:gamma_bound}), and proceed by observing
that 
\begin{equation}
\frac{\inprob-\outprob}{2}\error\le\left\langle \Yhat-\Ystar,\Adj-\E\Adj\right\rangle \overset{(i)}{\le}2\sup_{\mathbf{Y}\succeq0,\textmd{diag}(\mathbf{Y})\le\onevec}\left|\left\langle \mathbf{Y},\Adj-\E\Adj\right\rangle \right|,\label{eq:gro_arg}
\end{equation}
where step $(i)$ follows from the triangle inequality and the feasibility
of $\Yhat$ and $\Ystar$. Therefore, this argument reduces the problem
to bounding the RHS of~(\ref{eq:gro_arg}), which can be done using
the celebrated Grothendieck's inequality. One can already see at this
point that this approach yields sub-optimal bounds. For example, SDP
is known to achieve exact recovery ($\error=0$) under certain conditions,
yet the inequality~(\ref{eq:gro_arg}) can never guarantee a zero
$\error$. Sub-optimality arises in step $(i)$: the quantity $\left\langle \Yhat-\Ystar,\Adj-\E\Adj\right\rangle $
diminishes when $\Yhat-\Ystar$ is small, but the triangle inequality
and the worse-case bound used in~$(i)$ are too crude to capture
such behaviors. In comparison, our proof takes advantage of the structures
of the error matrix $\Yhat-\Ystar$ and its interplay with the noise
matrix $\Adj-\E\Adj$.

\textbf{Bounding the $S_{1}$ term: }A common approach involves using
the generalized Holder's inequality $S_{1}=\left\langle \Yhat-\Ystar,\PT(\Adj-\E\Adj)\right\rangle \le\error\norm[\PT(\Adj-\E\Adj)]{\infty}$.
Under SBM, one can show that $\norm[\PT(\Adj-\E\Adj)]{\infty}\lesssim\sqrt{\frac{\inprob\log\num^{2}}{\size}}$
with high probability, hence yielding the bound $S_{1}\lesssim\error\sqrt{\frac{\inprob\log\num^{2}}{\size}}$.
Variants of this approach are in fact common (sometimes implicitly)
in the proofs of exact recovery for SDP~\cite{ames2012clustering,chen2012clustering,amini2014semidefinite,hajek2016achieving,ChenXu2016}.
However, when $\sqrt{\frac{\inprob\log\num^{2}}{\size}}\ge\inprob-\outprob$
(where exact recovery is impossible), applying this bound for $S_{1}$
to the inequality~(\ref{eq:gamma_S1_plus_S2}) would yield a vacuous
bound for $\error$ . In comparison, Proposition~\ref{prop:S1} gives
a strictly sharper bound~(\ref{eq:S1bound}), which correctly characterizes
the behaviors of $S_{1}$ beyond the exact recovery regime.

\textbf{Bounding the $S_{2}$ term: }Note that since $\PTperp(\Ystar)=0$,
we have the equality $S_{2}=\left\langle \PTperp(\Yhat),\Adj-\E\Adj\right\rangle $.
It is easy to show that the matrix $\PTperp(\Yhat)$ is positive semidefinite
and has diagonal entries at most $4$ (cf.\ Fact~\ref{fact:Yhp_infty}).
Therefore, one may again attempt to control $S_{2}$ using the Grothendieck's
inequality, which would yield the bound $S_{2}\le4\cdot g(\A-\E\A)$
for some function $g$ whose exact form is not important for now.
The bound~(\ref{eq:S2bound}) in Proposition~\ref{prop:S2} is much
stronger \textemdash{} it depends on $\error$, which is in turn proportional
to the trace of the matrix $\PTperp(\Yhat)$ (cf.\ Fact~\ref{fact:Yperp_vs_error}).

\subsection{Preliminaries and additional notation\label{sec:prelim}}

Recall that $\mathbf{U}\in\mathbb{R}^{\num\times\numclust}$ is the
matrix of the left singular vectors of $\Ystar$. We observe that
$U_{ia}=1/\sqrt{\size}$ if node $i$ is in cluster $a$ and $U_{ia}=0$
otherwise. Therefore, $\mathbf{U}\mathbf{U}^{\top}$ is a block diagonal
matrix with all entries inside each diagonal block equal to $1/\size$. 

Define the ``noise'' matrix $\Noise\coloneqq\Adj-\E\Adj$. The matrix
$\Noise$ is symmetric, which introduces some minor dependency among
its entries. To handle this, we let $\HalfNoise$ be the matrix obtained
from $\Noise$ with its entries in the lower triangular part set to
zero. Note that $\Noise=\HalfNoise+\HalfNoise^{\top}$, and $\HalfNoise$
has independent entries (with zero entries considered $\Bern(0)$).
Similarly, we define $\HalfAdj$ as the upper triangular part of the
adjacency matrix $\Adj$. 

In the proof we frequently use the inequalities $\snr=\frac{(\inprob-\outprob)^{2}}{\inprob}<\inprob-\outprob\leq\inprob$.
Consequently, the assumption $\snr\geq\consts\numclust^{2}/\num$
implies that $\inprob\ge\consts\numclust^{2}/\num$. We also record
an elementary inequality that  is used multiple times.
\begin{lem}
\label{lm:simple_ineq}For any number $\alpha>0$, there exists a
number $C(\alpha)\ge1$ such that if $\snr\ge\frac{C(\alpha)\numclust}{\num}$,
then 
\[
\inprob e^{-\inprob\num/(\alpha\numclust)}\leq(\inprob-\outprob)e^{-\snr\num/(2\alpha\numclust)}.
\]
\end{lem}
\begin{proof}
Note that $\inprob\num\ge(\inprob-\outprob)\num\geq\snr\num\geq C(\alpha)\numclust.$
As long as $C(\alpha)$ is sufficiently large, we have $\frac{\inprob\num}{\numclust}\leq e^{\inprob\num/(2\alpha\numclust)}$.
These inequalities imply that 
\[
\frac{\inprob}{\inprob-\outprob}\leq\frac{\inprob\num}{\numclust}\leq e^{\inprob\num/(2\alpha\numclust)}\leq e^{(2\inprob-\snr)\num/(2\alpha\numclust)}.
\]
Multiplying both sides by $(\inprob-\outprob)e^{-2\inprob\num/(2\alpha\numclust)}$
yields the claimed inequality.
\end{proof}
Finally, we need a simple pilot bound, which ensures that the SDP
solution satisfies a non-trivial error bound $\error<\num^{2}$.
\begin{lem}
\emph{\label{thm:pilot_bound} }Under Model~\ref{mdl:SBM}, if $\inprob\geq\frac{1}{\num}$
then we have
\[
\error\leq45\sqrt{\frac{\num^{3}}{\snr}}
\]
with probability at least $1-P_{0}$, where $P_{0}:=2(e/2)^{-2\num}$.
In particular, if $\snr\ge\consts\frac{1}{\num}$, we have $\error\le\constu\num^{2}$
with high probability with $\constu=\frac{45}{\sqrt{\consts}}.$
\end{lem}
This theorem is a variation of Theorem 1.3 in \cite{Guedon2015},
and a special case of Theorem~1 in \cite{CMM} applied to the non-degree-corrected
setting. For completeness we provide the proof in Section~\ref{sec:proof_pilot_bound}.
The proof uses the Grothendieck's inequality, an approach pioneered
in \cite{Guedon2015}.

\subsection{Proof of Proposition \ref{prop:S1}\label{sec:proof_prop_S1}}

In this section we prove Proposition \ref{prop:S1}, which controls
the quantity $S_{1}$. Using the symmetry of $\Yhat-\Ystar$ and the
cyclic invariance of the trace, we obtain the identity
\begin{align*}
S_{1} & =\left\langle \PT\left(\mathbf{\Noise}\right),\Yhat-\Ystar\right\rangle \\
 & =\left\langle \U\U^{\top}\mathbf{\Noise},\Yhat-\Ystar\right\rangle +\left\langle \mathbf{\Noise}\U\U^{\top},\Yhat-\Ystar\right\rangle -\left\langle \U\U^{\top}\mathbf{\Noise}\U\U^{\top},\Yhat-\Ystar\right\rangle \\
 & =2\left\langle \U\U^{\top}\mathbf{\Noise},\Yhat-\Ystar\right\rangle -\left\langle \U\U^{\top}\mathbf{\Noise}\U\U^{\top},\Yhat-\Ystar\right\rangle \\
 & =2\left\langle \U\U^{\top}\HalfNoise,\Yhat-\Ystar\right\rangle +2\left\langle \U\U^{\top}\HalfNoise^{\top},\Yhat-\Ystar\right\rangle -\left\langle \U\U^{\top}(\HalfNoise+\HalfNoise^{\top})\U\U^{\top},\Yhat-\Ystar\right\rangle \\
 & =2\left\langle \U\U^{\top}\HalfNoise,\Yhat-\Ystar\right\rangle +2\left\langle \U\U^{\top}\HalfNoise^{\top},\Yhat-\Ystar\right\rangle -2\left\langle \U\U^{\top}\HalfNoise\U\U^{\top},\Yhat-\Ystar\right\rangle \\
 & =2\left\langle \U\U^{\top}\HalfNoise,\Yhat-\Ystar\right\rangle +2\left\langle \U\U^{\top}\HalfNoise^{\top},\Yhat-\Ystar\right\rangle -2\left\langle \U\U^{\top}\HalfNoise,(\Yhat-\Ystar)\U\U^{\top}\right\rangle .
\end{align*}
It follows that
\begin{align}
S_{1} & \leq2\left|\left\langle \U\U^{\top}\HalfNoise,\Yhat-\Ystar\right\rangle \right|+2\left|\left\langle \U\U^{\top}\HalfNoise^{\top},\Yhat-\Ystar\right\rangle \right|+2\left|\left\langle \U\U^{\top}\HalfNoise,(\Yhat-\Ystar)\U\U^{\top}\right\rangle \right|.\label{eq:three_terms}
\end{align}
Note that $\norm[\Yhat-\Ystar]{\infty}\leq1$ since $\Yhat,\Ystar\in\{0,1\}^{\num\times\num}$.
One can also check that $\norm[\left(\Yhat-\Ystar\right)\mathbf{U}\mathbf{U}^{\top}]{\infty}\leq1$.
This is due to the fact that each entry of the matrix inside the norm
is the mean of $\size$ entries in $\Yhat-\Ystar$ given the block
diagonal structure of $\mathbf{U}\mathbf{U}^{\top}$, and the absolute
value of the mean does not exceed $1$. With the same reasoning, we
see that $\norm[(\Yhat-\Ystar)\mathbf{U}\mathbf{U}^{\top}]1\le\norm[\Yhat-\Ystar]1=\error$. 

Key to our proof is a bound on the sum of order statistics. Intuitively,
given $m$ i.i.d. random variables, the sum of the $\beta$ largest
of them (in absolute value) scales as $O\left(\beta\sqrt{\log(m/\beta)}\right)$.
The following lemma, proved in Section \ref{sec:proof_lm_order_stat},
makes the above intuition precise and moreover establishes a uniform
bound in $\beta$.
\begin{lem}
\label{lm:order_stats} Let $m\ge8$ and $g\ge1$ be positive integers,
$\constu>0$ is a sufficiently small constant. For each $j\in[m]$,
define $X_{j}\coloneqq\sum_{i=1}^{g}(B_{ij}-\E B_{ij})$, where $B_{ij}$
are independent Bernoulli variables with variance at most $\rho$.
Then for a constant $D=\sqrt{7}$, we have 
\[
\sum_{j=1}^{\left\lceil \beta\right\rceil }\left|X_{(j)}\right|\le D\left\lceil \beta\right\rceil \sqrt{g\rho\log\left(m/\beta\right)},\quad\forall\beta\in(e^{-g\rho/\conste}m,\left\lceil \constu m\right\rceil ],
\]
with probability at least $1-P_{1}(m)$, where $P_{1}(m)\leq\frac{3}{m}$.
\end{lem}
\noindent We are ready to bound the RHS of equation (\ref{eq:three_terms})
and hence prove Proposition \ref{prop:S1}. Consider the event 
\begin{align*}
\mathcal{E}_{1}:= & \Bigg\{\left|\langle\U\U^{\top}\HalfNoise,\M\rangle\right|\leq2Db\sqrt{\frac{\inprob\log\left(\num^{2}/b\right)}{\size}},\\
 & \qquad\forall\M,b:\norm[\M]{\infty}\le1,\norm[\M]1\le b,\num^{2}e^{-\snr\num/\left(2\conste\numclust\right)}<b\le\constu\num^{2}\Bigg\},
\end{align*}
and let $\mathcal{E}_{2}$ be defined similarly with $\HalfNoise$
replaced by $\HalfNoise^{\top}$. We will use Lemma \ref{lm:order_stats}
to show that $\mathcal{E}_{i}$ holds with probability at least $1-P_{1}:=1-P_{1}(\num\numclust)$
for each $i=1,2$. Taking this claim as given, we now condition on
the intersection of the events $\mathcal{E}_{1}$, $\mathcal{E}_{2}$
and the conclusion of Lemma~\ref{thm:pilot_bound}. Note that the
matrix $\Yhat-\Ystar$ satisfies $\norm[\Yhat-\Ystar]{\infty}\le1$,
and Lemma \ref{thm:pilot_bound} further ensures that $\error:=\norm[\Yhat-\Ystar]1\le\constu\num^{2}$
for $\constu$ sufficiently small. If $\error\le\num^{2}e^{-\snr\num/\left(2\conste\numclust\right)},$
then the first inequality in Proposition \ref{prop:S1} holds and
we are done. Otherwise, on the event $\mathcal{E}_{1}$, we are guaranteed
that 
\[
\left|\langle\U\U^{\top}\HalfNoise,\Yhat-\Ystar\rangle\right|\leq2D\error\sqrt{\frac{\inprob\log\left(\num^{2}/\error\right)}{\size}}.
\]
Since the matrix $\M:=\left(\Yhat-\Ystar\right)\mathbf{U}\mathbf{U}^{\top}$
satisfies $\norm[\M]{\infty}\le1$ and $\norm[\M]1\le\error$, we
have the bound
\[
\left|\left\langle \U\U^{\top}\HalfNoise,(\Yhat-\Ystar)\U\U^{\top}\right\rangle \right|\le2D\error\sqrt{\frac{\inprob\log\left(\num^{2}/\error\right)}{\size}}.
\]
Similarly, on the event $\mathcal{E}_{2}$, we have the bound 
\[
\left|\langle\U\U^{\top}\HalfNoise^{\top},\Yhat-\Ystar\rangle\right|\le2D\error\sqrt{\frac{\inprob\log\left(\num^{2}/\error\right)}{\size}}.
\]
Applying these estimates to the RHS of equation (\ref{eq:three_terms}),
we arrive at the bound $S_{1}\le12D\error\sqrt{\frac{\inprob\log\left(\num^{2}/\error\right)}{\size}}$,
which is second inequality in Proposition \ref{prop:S1}.\\

It remains to bound the probability of the event $\mathcal{E}_{1}$,
for which we shall use Lemma \ref{lm:order_stats}; the same arguments
apply to the event $\mathcal{E}_{2}$. Let us take a digression to
inspect the structure of the random matrix $\mathbf{V}:=\mathbf{U}\mathbf{U}^{\top}\HalfNoise$.
We treat each zero entry in $\HalfNoise$ as a $\Bern(0)$ random
variable with its mean subtracted and independent of all other entries.
Therefore, all entries of $\HalfNoise$ are independent. Since $\mathbf{U}\mathbf{U}^{\top}$
is a block diagonal matrix, $\mathbf{V}$ can be partitioned into
$\numclust$ submatrices of size $\size\times\num$ stacked vertically,
where rows within the same submatrix are identical and rows from different
submatrices are independent with each of their entries equal to $1/\size$
times the sum of $\size$ independent centered Bernoulli random variables.
To verify our observations, for $a\in[\numclust]$ we use $\calR_{a}\coloneqq\left\{ (a-1)\size+1,\ldots,a\size\right\} $
to denote the set of row indices of the $a$-th submatrix of $\mathbf{V}$.
Consider any $i\in\calR_{a}$, that is, any row index of the $a$-th
submatrix of $\mathbf{V}$. Then for all $j\in[\num]$, we have $V_{ij}=\sum_{t=1}^{\num}\left(\mathbf{U}\mathbf{U}^{\top}\right)_{it}\halfnoise_{tj}=\size^{-1}\sum_{u\in\calR_{a}}\halfnoise_{uj}$.
We see that we get the same random variable by varying $i$ within
$\calR_{a}$ while fixing $j$, but independent random variables by
fixing $i$ while varying $j$. 

Fix an index $i_{a}:=(a-1)\size+1\in\calR_{a}$ for each $a\in[\numclust]$.
Consider any $\num\times\num$ matrix $\M$ and number $b$ such that
$\norm[\M]{\infty}\le1$, $\norm[\M]1\le b$ and $\num^{2}e^{-\snr\num/\left(2\conste\numclust\right)}<b\le\constu\num^{2}$.
We can compute 
\begin{align*}
\left|\langle\mathbf{V},\M\rangle\right| & \leq\sum_{i=1}^{\num}\sum_{j=1}^{\num}\left|V_{ij}\right|\left|M_{ij}\right|\\
 & =\sum_{a=1}^{\numclust}\sum_{i\in\calR_{a}}\sum_{j=1}^{\num}\left|V_{ij}\right|\left|M_{ij}\right|\\
 & =\sum_{a\in[\numclust],j\in[\num]}\left|\size V_{i_{a}j}\right|\left[\sum_{i\in\calR_{a}}\frac{\left|M_{ij}\right|}{\size}\right],
\end{align*}
where the last step follows from the previously established fact that
$V_{ij}=V_{i_{a}j},\forall i\in\calR_{a}$. Recall that the $\num\numclust$
random variables $\{\size V_{i_{a}j},a\in[\numclust],j\in[\num]\}$
are independent; moreover, each $\size V_{i_{a}j}$ is the sum of
$\size$ independent Bernoulli variables with variance at most $\inprob$.
Let $X_{(t)}$ denote the element in these $\num\numclust$ random
variables with the $t$-th largest absolute value. Define the quantity
$w\coloneqq\norm[\M]1/\size=\sum_{a\in[\numclust],j\in[\num]}\sum_{i\in\calR_{a}}\left|M_{ij}\right|/\size$.
Since $\sum_{i\in\calR_{a}}\size^{-1}\left|M_{ij}\right|\le\norm[\M]{\infty}\leq1$
and $w\le b/\size$, we have
\begin{align}
\left|\langle\mathbf{V},\M\rangle\right| & \le\begin{cases}
\sum_{t=1}^{\left\lceil w\right\rceil }\left|X_{(t)}\right|\le\sum_{t=1}^{\left\lceil b/\size\right\rceil }\left|X_{(t)}\right|, & b/\size\ge1\\
\left|X_{(1)}\right|(b/\size), & b/\size<1.
\end{cases}\label{eq:sum_by_order_stat}
\end{align}
Here we use the fact that for any sequence of numbers $a_{1},a_{2},\ldots$
in $[0,1]$, $\sum_{t}\left|X_{t}\right|a_{t}\le\sum_{t=1}^{\sum_{i}a_{i}}\left|X_{(t)}\right|$.
Now note that $b/\size\in(\num\numclust e^{-\snr\num/\left(2\conste\numclust\right)},\constu\num\numclust]$,
which implies that $b/\size\in(\num\numclust e^{-\inprob\size/\conste},\constu\num\numclust]$
by Lemma \ref{lm:simple_ineq}. Applying Lemma \ref{lm:order_stats}
with $m=\num\numclust\ge8$, $g=\size$, $\rho=\inprob$ and $\beta=b/\size$,
we are guaranteed that with probability at least $1-P_{1}(\num\numclust)$,
\begin{align*}
\sum_{t=1}^{\left\lceil b/\size\right\rceil }\left|X_{(t)}\right| & \le D\left\lceil b/\size\right\rceil \sqrt{\size\inprob\log\left(\frac{\num\numclust}{b/\size}\right)}\qquad\text{and}\qquad\left|X_{(1)}\right|\le D\sqrt{\size\inprob\log(\num\numclust)}
\end{align*}
simultaneously for all relevant $b/\size$. On this event, we can
continue the inequality (\ref{eq:sum_by_order_stat}) to conclude
that 
\[
\left|\langle\mathbf{V},\M\rangle\right|\le\begin{cases}
D\left\lceil b/\size\right\rceil \sqrt{\size\inprob\log\left(\frac{\num\numclust}{b/\size}\right)}, & b/\size\ge1\\
D(b/\size)\sqrt{\size\inprob\log(\num\numclust)}, & b/\size<1
\end{cases}
\]
simultaneously for all relevant matrices $\M$. It is easy to see
that the last RHS is bounded by $2Db\sqrt{\frac{\inprob\log\left(\num\numclust/w\right)}{\size}}$
in either case, hence the event $\mathcal{E}_{1}$ holds. 

We remark that the box constraints in the SDP (\ref{eq:SDP1}) are
crucial to the above arguments, which are ultimately applied to the
matrix $\M=\Yhat-\Ystar$ . In particular, the box constraints ensure
that $\norm[\M]{\infty}=\norm[\Yhat-\Ystar]{\infty}\leq1$, which
allows us to establish the inequality (\ref{eq:sum_by_order_stat})
and apply the order statistics bound in Lemma \ref{lm:order_stats}.

\subsection{Proof of Proposition \ref{prop:S2}\label{sec:proof_prop_S2}}

We begin our proof by re-writing $S_{2}$ as 
\begin{equation}
S_{2}=\left\langle \PTperp(\Yhat),\Noise\right\rangle ,\label{eq:S2_begin}
\end{equation}
which holds since $\PTperp(\Ystar)=0$ by definition of the projection
$\PTperp$. We can relate the matrix $\PTperp(\Yhat)$ appeared above
to the quantity of interest, $\error=\norm[\Yhat-\Ystar]1$. In particular,
observe that 
\begin{align*}
\Tr\left\{ \PTperp\left(\Yhat\right)\right\}  & =\Tr\left\{ \left(\I-\U\U^{\top}\right)\left(\Yhat-\Ystar\right)\left(\I-\U\U^{\top}\right)\right\} \\
 & \overset{(i)}{=}\Tr\left\{ \left(\I-\U\U^{\top}\right)\left(\Yhat-\Ystar\right)\right\} \\
 & \overset{(ii)}{=}\Tr\left\{ \U\U^{\top}\left(\Ystar-\Yhat\right)\right\} ,
\end{align*}
where step $(i)$ holds since trace is invariant under cyclic permutations
and $\I-\U\U^{\top}$ is a projection matrix, and step $(ii)$ holds
since $\diag\left(\Yhat\right)=\diag\left(\Ystar\right)=\one$ by
feasibility of $\Yhat$ and $\Ystar$ to the SDP~(\ref{eq:SDP1}).
Recall that all the non-zero entries of $\U\U^{\top}$ are in its
diagonal block and equal to $1/\size$. Since the corresponding diagonal-block
entries of the matrix $\Ystar-\Yhat$ are non-negative, we have
\[
\Tr\left\{ \U\U^{\top}\left(\Ystar-\Yhat\right)\right\} =\sum_{(i,j):\ystar_{ij}=1}\frac{1}{\size}\cdot\left|\left(\Ystar-\Yhat\right)_{ij}\right|\le\frac{\error}{\size}.
\]
Combining these pieces gives
\begin{fact}
\label{fact:Yperp_vs_error}$\Tr\left\{ \PTperp\left(\Yhat\right)\right\} \leq\frac{\error}{\size}$.
\end{fact}
Equipped with this fact, we proceed to bound the quantity $S_{2}$
given by equation~(\ref{eq:S2_begin}). We consider separately the
dense case $\inprob\ge\constdense\frac{\log\num}{\num}$ and the sparse
case $\inprob\le\constdense\frac{\log\num}{\num}$, where $\constdense>0$
is a constant given in the statement of Lemma~\ref{lm:untrimmed_L1_bound}.

\subsubsection{The dense case}

First assume that $p\ge\constdense\frac{\log\num}{\num}$. In this
case bounding $S_{2}$ is relatively straightforward, as the graph
spectrum is well-behaved. We first recall that the nuclear norm $\norm[\PTperp(\Yhat)]*$
is defined as the sum of the singular values of the matrix $\PTperp(\Yhat)$.
Since $\Yhat\succeq0$ by feasibility, we have $\PTperp(\Yhat)=(\I-\U\U^{\top})\Yhat(\I-\U\U^{\top})\succeq0$,
whence $\norm[\PTperp(\Yhat)]*=\Tr\left\{ \PTperp\left(\Yhat\right)\right\} .$
Revisiting the expression~(\ref{eq:S2_begin}), we obtain that 
\[
S_{2}=\left\langle \PTperp(\Yhat),\Noise\right\rangle \le\Tr\left\{ \PTperp\left(\Yhat\right)\right\} \cdot\opnorm{\Noise}\le\frac{\error}{\size}\cdot\opnorm{\Noise},
\]
where the first inequality follows from the duality between the nuclear
and spectral norms, and the second inequality follows from Fact \ref{fact:Yperp_vs_error}.

It remains to control the spectral norm $\opnorm{\Noise}$ of the
centered adjacency matrix $\Noise:=\Adj-\E\Adj$. This can be done
in the following lemma, which is proved in Section~\ref{sec:proof_spectral_bound_A}
using standard tools from random matrix theory.
\begin{lem}
\label{lm:spectral_bound_A} We have $\opnorm{\Adj-\E\Adj}\leq8\sqrt{\inprob\num}+174\sqrt{\log\num}$
with probability at least $1-P_{2}$ where $P_{2}\coloneqq\num^{-2}$.
\end{lem}
Applying the lemma, we obtain that with probability at least $1-P_{2}$,
\[
S_{2}\le\left(8\sqrt{\inprob\num}+174\sqrt{\log\num}\right)\cdot\frac{\error}{\size}\le D_{2}\sqrt{\inprob\num}\cdot\frac{\error}{\size},
\]
where in the last step holds for some constant $D_{2}$ sufficiently
large under the assumption $\inprob\ge\constdense\frac{\log\num}{\num}$.
This completes the proof of Proposition \ref{prop:S2} in the dense
case.

\subsubsection{The sparse case}

Now suppose that $p\le\constdense\frac{\log\num}{\num}$. In this
case we can no longer use the arguments above, because some nodes
will have degrees far exceeding than their expectation $O(\inprob\num)$,
and $\opnorm{\Noise}$ will be dominated by these nodes and become
much larger than $\sqrt{\inprob\num}$. This issue is particularly
severe when $\inprob,\outprob\asymp\frac{1}{\num}$, in which case
the expected degree is a constant, yet the maximum node degree diverges.
Addressing this issue requires a new argument. In particular, we show
that the matrix $\Noise$ can be partitioned into two parts, where
the first part has a spectral norm bounded as desired, and the second
part involves only a small number of edges; the structure of the SDP
solution allows us to control the impact of the second part.

As before, to avoid the minor dependency due to symmetry, we focus
on the upper triangular parts $\HalfAdj$ and $\HalfNoise$ of the
matrices $\Adj$ and $\Noise:=\Adj-\E\Adj$, and later make use of
the relations $\Adj=\HalfAdj+\HalfAdj^{\top}$ and $\Noise=\HalfNoise+\HalfNoise^{\top}$.
We define the sets
\begin{align*}
\calVrow & \coloneqq\left\{ i\in[\num]:\sum_{j\in[\num]}(\HalfAdj)_{ij}>40\inprob\num\right\} ,\\
\calVcol & \coloneqq\left\{ j\in[\num]:\sum_{i\in[\num]}(\HalfAdj)_{ij}>40\inprob\num\right\} ,
\end{align*}
which are the nodes whose degrees (more specifically, row/column sums
w.r.t.\ the halved matrix~$\HalfAdj$) are large compared to their
expectation $O(\inprob\num)$. For any matrix~$\M\in\real^{\num\times\num}$,
we define the matrix $\widetilde{\M}$ such that 
\[
\widetilde{M}_{ij}=\begin{cases}
0, & \text{if \ensuremath{i\in\calVrow} or \ensuremath{j\in\calVcol}},\\
M_{ij}, & \text{otherwise.}
\end{cases}
\]
In other words, $\widetilde{\M}$ is obtained from $\M$ by ``trimming''
the rows/columns corresponding to nodes with large degrees. With this
notation, we can write $\HalfNoise=\widetilde{\HalfNoise}+(\HalfNoise-\widetilde{\HalfNoise})$,
which can be combined with the expression~(\ref{eq:S2_begin}) to
yield 
\begin{align}
S_{2} & =\langle\PTperp(\Yhat),\HalfNoise+\HalfNoise^{\top}\rangle\nonumber \\
 & =2\langle\PTperp(\Yhat),\HalfNoise\rangle\nonumber \\
 & =2\langle\PTperp(\Yhat),\widetilde{\HalfNoise}\rangle+2\langle\PTperp(\Yhat),\HalfNoise-\widetilde{\HalfNoise}\rangle\nonumber \\
 & \le2\Tr\left\{ \PTperp(\Yhat)\right\} \cdot\opnorm{\widetilde{\HalfNoise}}+2\langle\PTperp(\Yhat),\HalfNoise-\widetilde{\HalfNoise}\rangle,\label{eq:S2_sparse_decomposition}
\end{align}
where in the last step we use the fact that $\Yhat\succeq0$ and thus
$\PTperp(\Yhat)=\left(\I-\U\U^{\top}\right)\Yhat\left(\I-\U\U^{\top}\right)\succeq0$.

The first term in (\ref{eq:S2_sparse_decomposition}) can be controlled
using the fact that the spectrum of the trimmed matrix $\widetilde{\HalfNoise}$
is well-behaved, for any $\inprob$. In particular, we prove the following
lemma in Section~\ref{sec:proof_trimmed_matrix_bound} as a consequence
of known results.
\begin{lem}
\label{lm:trimmed_spectral_bound}For some absolute constant $C>0$,
we have $\opnorm{\widetilde{\HalfNoise}}\leq C\sqrt{\inprob\num}$,
with probability at least $1-P_{3}$, where $P_{3}:=e^{-(3-\ln2)2\num}+(2\num)^{-3}$.
\end{lem}
One shall compare the bound in Lemma~\ref{lm:trimmed_spectral_bound}
with that in Lemma~\ref{lm:spectral_bound_A}. After trimming, the
term $\sqrt{\log\num}$ in Lemma~\ref{lm:spectral_bound_A} (which
would dominate in the sparse regime) disappears, and the spectral
norm of $\widetilde{\HalfNoise}$ behaves similarly as a random matrix
with Gaussian entries. Such a bound is in fact standard in recent
work on spectral algorithms applied to sparse graphs with constant
expected degrees~\cite{ChinRaoVu15,keshavan2009matrix}. Typically
these algorithms proceed by first trimming the graph, and then running
standard spectral algorithms with the trimmed graph as the input.
We emphasize that in our case trimming is used \emph{only in the analysis};
our algorithm itself does not require any trimming, and the SDP~(\ref{eq:SDP1})
is applied to the original graph. As shown below, we are able to control
the contributions from what is not trimmed, namely the second term
in~(\ref{eq:S2_sparse_decomposition}). Such a bound is made possible
by leveraging the structures of the solution $\Yhat$ induced by the
box constraints of the SDP~(\ref{eq:SDP1})~\textemdash ~a manifest
of the regularization effect of SDP.

Turning to the second term in equation~(\ref{eq:S2_sparse_decomposition}),
we first note that the Holder's type inequality $\langle\PTperp(\Yhat),\HalfNoise-\widetilde{\HalfNoise}\rangle\le\Tr\left\{ \PTperp(\Yhat)\right\} \opnorm{\HalfNoise-\widetilde{\HalfNoise}}$
is no longer sufficient, as the residual matrix $\HalfNoise-\widetilde{\HalfNoise}$
may have large eigenvalues. Here it is crucial to use an important
property of the matrix $\PTperp(\Yhat)$, namely the fact that the
magnitudes of its entries are $O(1)$, a consequence of the constraint
$0\le\Y\le\OneMat$ in the SDP~(\ref{eq:SDP1}). More precisely,
we have the following bound, which is proved in Section~\ref{sec:proof_Yhp_infty}.
\begin{fact}
\label{fact:Yhp_infty}We have $\norm[\PTperp(\Yhat)]{\infty}\leq4$.
\end{fact}
Intuitively, this bound ensures that the mass of the matrix $\PTperp(\Yhat)$
is spread across its entries, so its eigen-space is unlikely to be
aligned with the top eigenvector of the random matrix $\HalfNoise-\widetilde{\HalfNoise}$,
which by definition concentrates in a few columns/rows indexed by
$\calVcol/\calVrow$. Consequently, the quantity $\langle\PTperp(\Yhat),\HalfNoise-\widetilde{\HalfNoise}\rangle$
is likely to be small. To make this precise, we start with the inequality
\begin{equation}
\langle\PTperp(\Yhat),\HalfNoise-\widetilde{\HalfNoise}\rangle\le\norm[\PTperp(\Yhat)]{\infty}\cdot\norm[\HalfNoise-\widetilde{\HalfNoise}]1\label{eq:untrimmed_holder}
\end{equation}
The following lemma bounds the $\size_{1}$ norm of the residual matrix
$\HalfNoise-\widetilde{\HalfNoise}$.
\begin{lem}
\label{lm:untrimmed_L1_bound}Suppose $\inprob\geq\constp/\num$ for
some sufficiently large constant $\constp>0$, and $\inprob\leq\constdense\log\num/\num$
for some sufficiently small constant $\constdense>0$. Then there
exist some constants $C,c'>0$ such that 
\[
\norm[\HalfNoise-\widetilde{\HalfNoise}]1\leq C\inprob\num^{2}e^{-\inprob\num/\conste}
\]
with probability at least $1-P_{4}$ where $P_{4}:=4e^{-c'\sqrt{\num}}$.
\end{lem}
We prove this lemma in Section~\ref{sec:proof_untrimmed_L1_bound}
to follow using tail bounds for sub-exponential random variables.

Equipped with the above bounds, we are ready to bound $S_{2}$. If
$\error\le\constgamma\num^{2}e^{-\snr\num/(2\conste\numclust)}$,
then the first inequality in Proposition \ref{prop:S2} holds and
we are done. It remains to consider the case with $\error>\constgamma\num^{2}e^{-\snr\num/(2\conste\numclust)}$,
which by Lemma~\ref{lm:simple_ineq} implies that $\error>\constgamma\num^{2}\cdot\frac{\inprob}{\inprob-\outprob}e^{-\inprob\num/\conste}.$
The inequalities~(\ref{eq:S2_sparse_decomposition}) and~(\ref{eq:untrimmed_holder})
together give
\[
S_{2}\le2\Tr\left\{ \PTperp(\Yhat)\right\} \cdot\opnorm{\widetilde{\HalfNoise}}+\norm[\PTperp(\Yhat)]{\infty}\cdot\norm[\HalfNoise-\widetilde{\HalfNoise}]1\le2\frac{\error}{\size}\opnorm{\widetilde{\HalfNoise}}+4\norm[\HalfNoise-\widetilde{\HalfNoise}]1,
\]
where we use Facts~\ref{fact:Yperp_vs_error} and~\ref{fact:Yhp_infty}
in the second step. Applying Lemma~\ref{lm:trimmed_spectral_bound}
to the first RHS term above and Lemma~\ref{lm:untrimmed_L1_bound}
to the second, we obtain that
\[
S_{2}\le2C\frac{\sqrt{\inprob\num}}{\size}\error+4C\inprob\num^{2}e^{-\inprob\num/\conste}
\]
with probability at least $1-P_{3}-P_{4}$. Since $\error>\constgamma\num^{2}\cdot\frac{\inprob}{\inprob-\outprob}e^{-\inprob\num/\conste}$
as shown above , we obtain that
\[
S_{2}\le2C\frac{\sqrt{\inprob\num}}{\size}\error+\frac{4C}{\constgamma}(\inprob-\outprob)\error\le2C\frac{\sqrt{\inprob\num}}{\size}\error+\frac{1}{8}(\inprob-\outprob)\error,
\]
where the last step holds provided that the constant $\constgamma$
is large enough. This proves the second inequality in Proposition~\ref{prop:S2}.

\subsection{Proof of Lemma~\ref{lm:untrimmed_L1_bound}\label{sec:proof_untrimmed_L1_bound}}

For any matrix $\M$, let $\pos_{i\bullet}(\M)$ and $\pos_{\bullet j}(\M)$
denote the number of positive entries in $i$-th row and $j$-th column
of $\M$, respectively. Recall that the entries of $\HalfAdj$ are
independent Bernoulli random variables with variance at most $\inprob$,
where $\constp/\num\le\inprob\le\constdense\log\num/\num$ for some
sufficiently large constant $\constp>0$ and sufficiently small constant
$c_{\inprob}>0$. By definition, we have $\HalfNoise\coloneqq\HalfAdj-\E\HalfAdj$,
whence
\begin{align*}
\norm[\HalfNoise-\widetilde{\HalfNoise}]1 & =\norm[(\HalfAdj-\E\HalfAdj)-\widetilde{(\HalfAdj-\E\HalfAdj)}]1\\
 & =\norm[(\HalfAdj-\E\HalfAdj)-(\widetilde{\HalfAdj}-\widetilde{\E\HalfAdj})]1\\
 & \leq\norm[\E\HalfAdj-\widetilde{\E\HalfAdj}]1+\norm[\HalfAdj-\widetilde{\HalfAdj}]1.
\end{align*}
The two terms on the last RHS are both random quantities depending
on the sets $\calVrow=\{i:\pos_{i\bullet}(\HalfAdj)\ge40\inprob\num\}$
and $\calVcol:=\{j:\pos_{\bullet j}(\HalfAdj)\ge40\inprob\num\}$.
We bound these two terms separately.

\subsubsection{Bounding $\protect\norm[\protect\E\protect\HalfAdj-\widetilde{\protect\E\protect\HalfAdj}]1$}

Let $\indic(\cdot)$ denote the indicator function. We begin by investigating
the indicator variable $G_{i}:=\indic(i\in\calVrow)=\indic(\pos_{i\bullet}(\HalfAdj)\ge40\inprob\num)$;
similar properties hold for $G'_{j}:=\indic(j\in\calVcol)$. Note
that $\{G_{i},i\in[\num]\}$ are independent Bernoulli random variables.
To bound their rates $\E G_{i}$, we observe that each $\pos_{i\bullet}(\HalfAdj)=\sum_{j}\halfadj_{ij}$
is the sum of $\num$ independent Bernoulli random variables with
variance at most $\inprob(1-\inprob)\num$, and $\E\pos_{i\bullet}(\HalfAdj)\le\inprob\num$.
The Bernstein's inequality ensures that for any number $z>\inprob\num$,
\begin{align}
\P\left\{ \pos_{i\bullet}(\HalfAdj)\geq z\right\}  & \le\P\left\{ \pos_{i\bullet}(\HalfAdj)-\E\pos_{i\bullet}(\HalfAdj)\geq z-\inprob\num\right\} \nonumber \\
 & \leq\exp\left\{ -\frac{\frac{1}{2}(z-\inprob\num)^{2}}{\inprob(1-\inprob)\num+\frac{1}{3}\cdot1\cdot(z-\inprob\num)}\right\} \nonumber \\
 & \leq\exp\left\{ -\frac{(z-\inprob\num)^{2}}{2z}\right\} .\label{eq:degree_bound}
\end{align}
Plugging in $z=40\inprob\num$ yields
\begin{align*}
\E G_{i}=\P\left\{ \pos_{i\bullet}(\HalfAdj)\geq40\inprob\num\right\}  & \leq\exp\left\{ -\frac{(40\inprob\num-\inprob\num)^{2}}{2\cdot40\inprob\num}\right\} \leq e^{-8\inprob\num}.
\end{align*}

We now turn to the quantity of interest, $\norm[\E\HalfAdj-\widetilde{\E\HalfAdj}]1$.
Since $\E\halfadj_{ij}\le\inprob$ for all $(i,j)$, we have that
\begin{align*}
\norm[\E\HalfAdj-\widetilde{\E\HalfAdj}]1 & =\sum_{i}\sum_{j}\E\halfadj_{ij}\cdot\indic(i\in\calVrow\text{ or }j\in\calVcol)\\
 & \leq\sum_{i}\sum_{j}\inprob\cdot\indic(i\in\calVrow)+\sum_{i}\sum_{j}\inprob\cdot\indic(j\in\calVcol)\\
 & =\inprob\num\sum_{i}G_{i}+\inprob\num\sum_{j}G_{j}'.
\end{align*}
The quantity $\sum_{i}G_{i}$ is the sum of independent Bernoulli
variables with rates bounded above. Applying the Bernstein's inequality
to this sum, we obtain that 
\begin{align*}
\P\left\{ \sum_{i}G_{i}>2\num e^{-\inprob\num/\conste}\right\}  & \leq\P\left\{ \sum_{i}G_{i}-\sum_{i}\E G_{i}>2\num e^{-\inprob\num/\conste}-\num e^{-\inprob\num/\conste}\right\} \\
 & \leq\exp\left\{ -\frac{\frac{1}{2}\left(\num e^{-\inprob\num/\conste}\right)^{2}}{\num\cdot e^{-\inprob\num/\conste}+\frac{1}{3}\cdot1\cdot\num e^{-\inprob\num/\conste}}\right\} \\
 & \leq\exp\left\{ -\frac{3}{8}\num e^{-\inprob\num/\conste}\right\} \leq\exp\left\{ -\frac{3}{8}\sqrt{\num}\right\} ,
\end{align*}
where the first two steps hold because $\conste\ge28$,  and the
last step holds by the assumption in Lemma~\ref{lm:untrimmed_L1_bound}
that $\inprob\leq\constdense\log\num/\num$ for some sufficiently
small $\constdense>0.$ The same tail bound holds for the sum $\sum_{j}G_{j}'$.
It follows that with probability at least $1-2e^{-\frac{3}{8}\sqrt{\num}}$,
we have
\[
\norm[\E\HalfAdj-\widetilde{\E\HalfAdj}]1\leq4\inprob\num^{2}e^{-\inprob\num/\conste}
\]

\subsubsection{Bounding $\protect\norm[\protect\HalfAdj-\widetilde{\protect\HalfAdj}]1$}

Since the matrix $\HalfAdj-\widetilde{\HalfAdj}$ have non-negative
entries, we have the inequality 
\begin{align}
\norm[\HalfAdj-\widetilde{\HalfAdj}]1 & =\sum_{i}\sum_{j}\halfadj_{ij}\indic(i\in\calVrow\text{ or }j\in\calVcol)\nonumber \\
 & \leq\sum_{i}\underbrace{\sum_{j}\halfadj_{ij}\indic(i\in\calVrow)}_{Z_{i}}+\sum_{j}\underbrace{\sum_{i}\halfadj_{ij}\indic(j\in\calVcol)}_{Z_{j}'}.\label{eq:untrimmed_rand_bound}
\end{align}
The variables $\{Z_{i},i\in[\num]\}$ are independent. Below we show
that they are sub-exponential, for which we recall that a variable
$X$ is called sub-exponential with parameter $\lambda$ if 
\begin{equation}
\E e^{t(X-\E X)}\leq e^{t^{2}\lambda^{2}/2},\qquad\text{for all }t\in\real\text{ with }\left|t\right|\leq\frac{1}{\lambda}.\label{eq:sub-exp-mgf}
\end{equation}
It is a standard fact~\cite{Vershynin-non-asymp-bounds} that sub-exponential
variables can be equivalently defined in terms of the tail condition
\begin{equation}
\P\left\{ \left|X\right|\geq z\right\} \le2e^{-8z/\lambda},\quad\text{for all \ensuremath{z\ge0}}.\label{eq:sub-exp-tail}
\end{equation}
For the sake of completeness and tracking explicit constant values,
we prove in Section~\ref{sec:proof_sub_exp_eqv_defs} the one-sided
implication (\ref{eq:sub-exp-tail})$\Rightarrow$(\ref{eq:sub-exp-mgf}),
which is what we need below. 

To verify the condition~(\ref{eq:sub-exp-tail}) for each $Z_{i}$,
we observe that by definition either $Z_{i}=0$ or $Z_{i}\geq40\inprob\num$.
Therefore, for each number $z\geq40\inprob\num$, we have the tail
bound
\begin{align*}
\P\left\{ Z_{i}\geq z\right\}  & =\P\left\{ \pos_{i\bullet}(\HalfAdj)\geq z\right\} \\
 & \overset{(i)}{\le}\exp\left\{ -\frac{(z-\inprob\num)^{2}}{2z}\right\} \\
 & \leq\exp\left\{ -\frac{(z-z/40)^{2}}{2z}\right\} \\
 & \leq e^{-z/4},
\end{align*}
where in step $(i)$ we use the previously established bound~(\ref{eq:degree_bound}).
For each $0<z<40\inprob\num$, we use the bound~(\ref{eq:degree_bound})
again to get
\begin{align*}
\P\left\{ Z_{i}\geq z\right\}  & =\P\left\{ Z_{i}\geq40\inprob\num\right\} \\
 & =\P\left\{ \pos_{i\bullet}(\HalfAdj)\geq40\inprob\num\right\} \\
 & \leq e^{-8\inprob\num}\leq e^{-z/4}.
\end{align*}
We conclude that $\P\left\{ Z_{i}\geq z\right\} \leq e^{-z/4}$ for
all $z\ge0$. Since $Z_{i}$ is non-negative, it satisfies the tail
condition~(\ref{eq:sub-exp-tail}) with $\lambda=32$ and hence is
sub-exponential as claimed. Moreover, the non-negative variable $Z_{i}$
satisfies the expectation bound
\begin{align*}
\E Z_{i} & =\int_{0}^{20\inprob\num}\P\left\{ Z_{i}>z\right\} \textup{d}z+\int_{20\inprob\num}^{\infty}\P\left\{ Z_{i}>z\right\} \textup{d}z\\
 & \overset{(i)}{\le}\int_{0}^{20\inprob\num}e^{-8\inprob\num}\textup{d}z+\int_{20\inprob\num}^{\infty}e^{-z/4}\textup{d}z\\
 & =20\inprob\num\cdot e^{-8\inprob\num}+4e^{-5\inprob\num}\\
 & \overset{(ii)}{\le}24\inprob\num e^{-5\inprob\num},
\end{align*}
where step $(i)$ follows the previously established tail bounds for
$Z_{i}$, and step $(ii)$ follows from the assumption that $\inprob\geq\constp/\num$
for some sufficiently large constant $\constp>0$. It follows that
$\E\left[\sum_{i}Z_{i}\right]\leq24\inprob\num^{2}e^{-5\inprob\num}.$

To control the sum $\sum_{i}Z_{i}$ in equation~(\ref{eq:untrimmed_rand_bound}),
we use a Bernstein-type inequality for the sum of sub-exponential
random variables.
\begin{lem}[Theorem 1.13 of \cite{rigollet201518}]
\label{lm:subexp_bernstein} Suppose that $X_{1},\ldots,X_{\num}$
are independent random variables, each of which is sub-exponential
with parameter $\lambda$ in the sense of~(\ref{eq:sub-exp-mgf}).
Then for each $t\geq0$, we have
\[
\P\left\{ \sum_{i=1}^{\num}(X_{i}-\E X_{i})\geq t\right\} \leq\exp\left\{ -\frac{1}{2}\min\left(\frac{t^{2}}{\lambda^{2}\num},\frac{t}{\lambda}\right)\right\} .
\]
\end{lem}
For a sufficiently large constant $C_{2}>1$, we apply the lemma above
to obtain the tail bound
\begin{align*}
\P\left\{ \sum_{i}Z_{i}\geq2C_{2}\inprob\num^{2}e^{-5\inprob\num}\right\} \overset{(i)}{\le} & \P\left\{ \sum_{i}\left(Z_{i}-\E Z_{i}\right)\geq C_{2}\inprob\num^{2}e^{-5\inprob\num}\right\} \\
\overset{(ii)}{\le} & \exp\left\{ -\frac{1}{2}\min\left(\frac{(C_{2}\inprob\num^{2}e^{-5\inprob\num})^{2}}{(32)^{2}\num},\frac{C_{2}\inprob\num^{2}e^{-5\inprob\num}}{32}\right)\right\} ,
\end{align*}
where step $(i)$ follows from our previous bound on $\E\left[\sum_{i}Z_{i}\right]$.
To proceed, we recall the assumption in Lemma~\ref{lm:untrimmed_L1_bound}
that $\inprob$ satisfies $C_{p}/n\le\inprob\leq\constdense\log\num/\num$
for $C_{p}$ sufficiently large and $\constdense>0$ sufficiently
small, which implies that 
\begin{align*}
\log\left(C_{2}\inprob\num^{2}e^{-5\inprob\num}\right) & =\log C_{2}+\log(\inprob\num)+\log\num-5\inprob\num\geq\frac{3}{4}\log\num,
\end{align*}
whence $C_{2}\inprob\num^{2}e^{-5\inprob\num}\geq\num^{3/4}$. Plugging
into the above tail bound, we get that
\begin{align*}
\P\left\{ \sum_{i}Z_{i}\geq2C_{2}\inprob\num^{2}e^{-5\inprob\num}\right\} \le & \exp\left\{ -\frac{1}{2}\min\left(\frac{n^{1/2}}{(32)^{2}},\frac{n^{3/4}}{32}\right)\right\} \leq e^{-c'\sqrt{\num}},
\end{align*}
for some absolute constant $c'$. The same tail bound holds for the
sum $\sum_{j}Z_{j}'$. Combining these two bounds with the inequality~(\ref{eq:untrimmed_rand_bound}),
we obtain that
\[
\norm[\HalfAdj-\widetilde{\HalfAdj}]1\leq4C_{2}\inprob\num^{2}e^{-5\inprob\num}
\]
with probability at least $1-2e^{-c'\sqrt{n}}$.\\

Putting together the above bounds on $\norm[\E\HalfAdj-\widetilde{\E\HalfAdj}]1$
and $\norm[\HalfAdj-\widetilde{\HalfAdj}]1$, we conclude that there
exists an absolute constant $C>0$ such that
\[
\norm[\HalfNoise-\widetilde{\HalfNoise}]1\leq\norm[\E\HalfAdj-\widetilde{\E\HalfAdj}]1+\norm[\HalfAdj-\widetilde{\HalfAdj}]1\leq C\inprob\num^{2}e^{-\inprob\num/\conste}
\]
with probability at least $1-4e^{-c'\sqrt{n}}$, as desired.

\section{Proof of Theorem~\ref{thm:exp_rate_robust}\label{sec:proof_exp_rate_robust}}

We only need to show that the conclusion of Theorem~\ref{thm:exp_rate}
(the bound on $\norm[\Yhat-\Ystar]1$) holds under the Heterogeneous
SBM in Model~\ref{mdl:heterSBM}. Once this is established, the conclusion
in Theorem~\ref{thm:cluster_error_rate} follows immediately, as
the clustering error is proportional to $\norm[\Yhat-\Ystar]1$ (see
Proposition~\ref{sec:kmedians}).

To achieve the goal above, we first consider a model that bridges
Model~\ref{mdl:SBM} and Model~\ref{mdl:heterSBM}, and prove robustness
under this intermediate model.

\begin{mdl}[Semi-random Stochastic Block Model]\label{mdl:semiSBM}Given
an \emph{arbitrary} set of node pairs $\pairset\subseteq\{(i,j)\in[\num]\times[\num]:i<j\}$
and the ground-truth clustering $\LabelStar$ (encoded in the matrix
$\Ystar$), a graph is generated as follows. For each $(i,j)\in\pairset$,
$\adj_{ij}=\ystar_{ij}$ deterministically. For each $(i,j)\notin\pairset$,
$\adj_{ij}$ is generated according to the standard SBM in Model~\ref{mdl:SBM};
that is, $\adj_{ij}\sim\Bern(\inprob)$ if $\ystar_{ij}=1$, and $\adj_{ij}\sim\Bern(\outprob)$
if $\ystar_{ij}=0$. \end{mdl}

Note that the standard SBM in Model~\ref{mdl:SBM} is a special case
of the Semi-random SBM in Model~\ref{mdl:semiSBM} with an empty
set $\pairset$. Model~\ref{mdl:semiSBM} is in turn a special case
of the Heterogeneous SBM in Model~\ref{mdl:heterSBM} where $\adj_{ij}\sim\Bern(\indic(\ystar_{ij}=1))$
for each $(i,j)\in\pairset$, and $A_{ij}\sim\Bern\left(\inprob\indic(\ystar_{ij}=1)+\outprob\indic(\ystar_{ij}=0)\right)$
for each $(i,j)\notin\pairset$.

We first show that the conclusion of Theorem~\ref{thm:exp_rate}
holds under Model~\ref{mdl:semiSBM}. We do so by verifying the steps
of the proof of Theorem~\ref{thm:exp_rate} given in Section~\ref{sec:proof_exp_rate}.
The proof consists of Fact~\ref{fact:expectation_gamma} concerning
the expected graph $\E\Adj$, and Propositions~\ref{prop:S1} and~\ref{prop:S2},
which bound the deviation terms $S_{1}$ and $S_{2}$. Under Model~\ref{mdl:semiSBM},
Fact~\ref{fact:expectation_gamma} remains valid: each entry of $\Adj$
in the set $\pairset$ is changed in the direction of the sign of
the corresponding entries of $\Yhat-\Ystar$, which only increases
the LHS of equation~(\ref{eq:expectation_gamma}). Both terms $S_{1}$
and $S_{2}$ involve the noise matrix $\Noise:=\Adj-\E\Adj$. Examining
the proofs of Propositions~\ref{prop:S1} and \ref{prop:S2}, we
see that they only rely on the independence of the entries of $\HalfNoise\coloneqq\HalfAdj-\E\HalfAdj$
(that is, the upper triangular entries of $\Noise:=\Adj-\E\Adj$)
as well as upper bounds on their expectations and variances. Under
Model~\ref{mdl:semiSBM}, the entries of $\Adj$ indexed by $\pairset$
are deterministically set to $1$ or $0$; the corresponding entries
of $\HalfNoise$ become $\Bern(0)$, in which case they remain independent,
with their expectations and variances decreased to zero.\footnote{We illustrate this argument using Section \ref{sec:proof_untrimmed_L1_bound}
as an example. There we would like to upper bound the quantity $\norm[\HalfNoise-\widetilde{\HalfNoise}]1$.
Suppose that under Model~\ref{mdl:semiSBM}, the entry $(i,j)\in\pairset$
is such that $\adj_{ij}=\ystar_{ij}=1$ surely. In this case $\Psi_{ij}\sim\Bern(0)$,
which can be written as $\Psi_{ij}=\halfadj_{ij}-\E\halfadj_{ij}$
with $\E\halfadj_{ij}=0\le\inprob$ and $\Var(\halfadj_{ij})=0\leq\inprob(1-\inprob)$.
This is all that is required when we proceed to bound $\norm[\E\HalfAdj-\widetilde{\E\HalfAdj}]1$
and $\norm[\HalfAdj-\widetilde{\HalfAdj}]1$.} Therefore, Propositions~\ref{prop:S1} and \ref{prop:S2}, and hence
all the steps in the proof of Theorem~\ref{thm:exp_rate}, continue
to hold under Model~\ref{mdl:semiSBM}.

We now turn to the Heterogeneous SBM in Model~\ref{mdl:heterSBM},
and show that this model can be reduced to Model~\ref{mdl:semiSBM}
via a coupling argument. Suppose that under Model~\ref{mdl:heterSBM},
$\adj_{ij}\sim\Bern(\inprob_{ij})$ for each $(i,j)$ with $\ystar_{ij}=1$,
and $\adj_{ij}\sim\Bern(\outprob_{ij})$ for each $(i,j)$ with $\ystar_{ij}=0$,
where by assumption $\inprob_{ij}\ge\inprob$ and $\outprob_{ij}\le q$.
Model~\ref{mdl:heterSBM} is equivalent to the following 3-step process:
\begin{enumerate}
\item[Step 1:] We first generate a set of edge pairs $\pairset$ as follows: independently
for each $(i,j),i<j$, if $\ystar_{ij}=1$, then $(i,j)$ is included
in $\pairset$ with probability $1-\frac{1-\inprob_{ij}}{1-\inprob}$;
if $\ystar_{ij}=0$, then $(i,j)$ is included in $\pairset$ with
probability $1-\frac{\outprob_{ij}}{\outprob}$.
\item[Step 2:] Independently of above, we sample a graph from Model~\ref{mdl:SBM};
let $\A^{0}$ denote its adjacency matrix.
\item[Step 3:] The final graph $\A$ is constructed as follows: for each $(i,j)$,
$i<j$ with $\ystar_{ij}=1$, $\adj_{ij}=\indic\left(\adj_{ij}^{0}=1\text{ or }(i,j)\in\pairset\right)$;
for each $(i,j)$, $i<j$ with $\ystar_{ij}=0$, $\adj_{ij}=\indic\left(\adj_{ij}^{0}=1\text{ and }(i,j)\notin\pairset\right)$.
\end{enumerate}
Note that the assumption $\inprob_{ij}\ge\inprob$ and $\outprob_{ij}\le q$
ensures that the probabilities in step 1 are in $[0,1]$. One can
verify that the distribution of $\A$ is the same as in Model~\ref{mdl:heterSBM};
in particular, for each $(i,j)$, $i<j$, we have
\[
\P(\adj_{ij}=1)=\begin{cases}
\P\left(\adj_{ij}^{0}=1\text{ or }(i,j)\in\pairset\right)=1-(1-\inprob)\cdot\frac{1-\inprob_{ij}}{1-\inprob}=\inprob_{ij}, & \text{if }\ystar_{ij}=1,\\
\P\left(\adj_{ij}^{0}=1\text{ and }(i,j)\notin\pairset\right)=\outprob\cdot\frac{\outprob_{ij}}{\outprob}=\outprob_{ij}, & \text{if }\ystar_{ij}=0.
\end{cases}
\]
Now, conditioned on the set $\pairset$, the distribution of $\A$
follows Model~\ref{mdl:semiSBM}, since steps 1 and 2 are independent.
We have established above that under Model~\ref{mdl:semiSBM}, the
error bound in Theorem~\ref{thm:exp_rate} holds with high probability.
Integrating out the randomness of the set~$\pairset$ proves that
the error bound holds with the same probability under Model~\ref{mdl:heterSBM}.

\section*{Acknowledgement}

Y. Fei and Y. Chen were partially supported by the National Science
Foundation CRII award 1657420.

\appendix
\appendixpage

\section{Approximate $k$-medians algorithm\label{sec:kmedians}}

We describe the procedure for extracting an explicit clustering $\LabelHat$
from the solution $\Yhat$ of the SDP~(\ref{eq:SDP1}). Our procedure
builds on the idea in the work \cite{CMM}, and applies a version
of the $k$-median algorithm to the rows of $\Yhat$, viewed as $\num$
points in $\real^{\num}$. The $k$-median algorithm seeks a partition
of these $\num$ points into $k$ clusters, and a center associated
with each cluster, such that the sum of the $\ell_{1}$ distances
of each point to its cluster center is minimized. Note that this procedure
differs from the standard $k$-means algorithm: the latter minimizes
the sum of \emph{squared} distance, and uses the $\ell_{2}$ distance
rather than $\ell_{1}$. 

To formally specify the algorithm, we need some additional notations.
Let $\mbox{Rows}(\M)$ be the set of rows in the matrix $\M$. Define
$\mathbb{M}_{\num,\numclust}$ to be the set of membership matrices
corresponding to all $\numclust$-partitions of $[\num]$; that is,
$\mathbb{M}_{\num,\numclust}$ is the set of matrices in $\{0,1\}^{\num\times\numclust}$
such that each of their rows has exactly one entry equal to 1 and
each of their columns has at least one entry equal to 1. By definition
each element in $\mathbb{M}_{\num,\numclust}$ has exactly $\numclust$
distinct rows. The $k$-medians algorithm described above can be written
as an optimization problem: 
\begin{align}
\min_{\boldsymbol{\Psi},\mathbf{X}}\  & \norm[\boldsymbol{\Psi}\mathbf{X}-\Yhat]1\nonumber \\
\mbox{s.t.\ } & \boldsymbol{\Psi}\in\mathbb{M}_{\num,\numclust},\label{eq:k-median problem}\\
 & \mathbf{X}\in\mathbb{R}^{\numclust\times\num},\;\mbox{Rows}(\mathbf{X})\subset\mbox{Rows}(\Yhat)\nonumber 
\end{align}
Here the optimization variable $\boldsymbol{\Psi}$ encodes a partition
of $\num$ points, assigning the $i$-th point to the cluster indexed
by the unique non-zero element of the $i$-th row of $\boldsymbol{\Psi}$;
the rows of the variable $\mathbf{X}$ represent the $r$ cluster
centers. Note that the last constraint in~(\ref{eq:k-median problem})
stipulates that the cluster centers be elements of the input data
points represented by the rows of $\Yhat$, so this procedure is sometimes
called $k$-medoids. Let $(\overline{\boldsymbol{\Psi}},\overline{\mathbf{X}})$
be the optimal solution of the problem~(\ref{eq:k-median problem}). 

Finding the exact solution $(\overline{\boldsymbol{\Psi}},\overline{\mathbf{X}})$
is in general computationally intractable, but polynomial-time constant-factor
approximation algorithms exist. In particular, we will make use of
the polynomial-time procedure in the paper \cite{charikar1999constant},
which produces an approximate solution $(\check{\boldsymbol{\Psi}},\check{\mathbf{X}})\in\mathbb{M}_{\num,\numclust}\times\mathbb{R}^{\numclust\times\num}$
that is guaranteed to satisfy $\mbox{Rows}(\check{\mathbf{X}})\subset\mbox{Rows}(\Yhat)$
and 
\begin{equation}
\norm[\check{\boldsymbol{\Psi}}\check{\mathbf{X}}-\Yhat]1\leq\apxconst\norm[\overline{\boldsymbol{\Psi}}\overline{\mathbf{X}}-\Yhat]1\label{eq:k-median guarantee}
\end{equation}
with an approximation ratio $\apxconst=\frac{20}{3}$. The variable
$\check{\boldsymbol{\Psi}}$ encodes our final clustering assignment.

The complete procedure $\apkmedian$ is given in Algorithm~\ref{alg:apx_kmedian},
which takes as input a data matrix $\Yhat$ (which in our case is
the solution to the SDP (\ref{eq:SDP1})) and outputs an explicit
clustering $\LabelHat=\apkmedian(\Yhat)$. We assume that the number
of clusters $\numclust$ is known.

\begin{algorithm}

\caption{Approximate $k$-median $\protect\apkmedian$\label{alg:apx_kmedian}}

Input: data matrix $\Yhat\in\real^{\num\times\num}$; approximation
factor $\apxconst\ge1$ .

1. Use a $\apxconst$-approximation algorithm to solve the optimization
problem (\ref{eq:k-median problem}). Denote the solution as $(\check{\boldsymbol{\Psi}},\check{\mathbf{X}})$.

2. For each $i\in[\num]$, set $\labelhat_{i}\in[\numclust]$ to be
the index of the unique non-zero element of $\left\{ \check{\Psi}_{1i},\check{\Psi}_{2i},\ldots,\check{\Psi}_{\numclust i}\right\} $;
assign node $i$ to cluster $\labelhat_{i}$.

Output: Clustering assignment $\LabelHat\in[\numclust]^{\num}$.

\end{algorithm}

\section{Proof of Lemma \ref{thm:pilot_bound}\label{sec:proof_pilot_bound}}

Starting with the previous established inequality (\ref{eq:gamma_bound}),
we have the bound $\error\le\frac{2}{\inprob-\outprob}\left\langle \Yhat-\Ystar,\Adj-\E\Adj\right\rangle .$
To control the RHS, we compute 
\begin{align*}
\left\langle \Yhat-\Ystar,\Adj-\E\Adj\right\rangle  & =\left\langle \Yhat,\Adj-\E\Adj\right\rangle -\left\langle \Ystar,\Adj-\E\Adj\right\rangle \\
 & \leq2\sup_{\mathbf{Y}\succeq0,\textmd{diag}(\mathbf{Y})\le\onevec}\left|\left\langle \mathbf{Y},\Adj-\E\Adj\right\rangle \right|.
\end{align*}
The Grothendieck's inequality~\cite{Grothendieck,Lindenstrauss}
guarantees that 
\[
\sup_{\mathbf{Y}\succeq0,\textmd{diag}(\mathbf{Y})\le\onevec}\left|\left\langle \mathbf{Y},\Adj-\E\Adj\right\rangle \right|\leq K_{G}\norm[\Adj-\E\Adj]{\infty\to1}
\]
where $K_{G}$ denotes the Grothendieck's constant ($0<K_{G}\leq1.783$)
and $\norm[\M]{\infty\to1}\coloneqq\sup_{\mathbf{x}:\|\mathbf{x}\|_{\infty}\leq1}\norm[\M\mathbf{x}]1$
is the $\size_{\infty}\to\size_{1}$ operator norm for a matrix $\M$.
Furthermore, we have the identity 
\[
\norm[\Adj-\E\Adj]{\infty\to1}=\sup_{\mathbf{x}:\|\mathbf{x}\|_{\infty}\leq1}\norm[(\Adj-\E\Adj)\mathbf{x}]1=\sup_{\mathbf{y},\mathbf{z}\in\{\pm1\}^{\num}}\left|\mathbf{y}^{\top}(\Adj-\E\Adj)\mathbf{z}\right|.
\]

Set $\sigma^{2}:=\sum_{1\leq i<j\leq\num}\Var(\adj_{ij})$. For each
pair of fixed vectors $\mathbf{y},\mathbf{z}\in\{\pm1\}^{\num}$,
the Bernstein inequality ensures that for each number $t\geq0$, 
\[
\Pr\left\{ \left|\mathbf{y}^{\top}(\Adj-\E\Adj)\mathbf{z}\right|>t\right\} \leq2\exp\left\{ -\frac{t^{2}}{2\sigma^{2}+4t/3}\right\} .
\]
Setting $t=\sqrt{16\num\sigma^{2}}+\frac{8}{3}\num$ gives
\[
\Pr\left\{ \left|\mathbf{y}^{\top}(\Adj-\E\Adj)\mathbf{z}\right|>\sqrt{16\num\sigma^{2}}+\frac{8}{3}\num\right\} \leq2e^{-2\num}.
\]
Applying the union bound and using the fact that $\sigma^{2}\leq\inprob(\num^{2}-\num)/2$,
we obtain that with probability at most $2^{2\num}\cdot2e^{-2\num}=2(e/2)^{-2\num}$,
\[
\norm[\Adj-\E\Adj]{\infty\to1}>2\sqrt{2\inprob(\num^{3}-\num^{2})}+\frac{8}{3}\num.
\]

Combining pieces, we conclude that with probability at least $1-2(e/2)^{-2\num}$,
\[
\left\langle \Yhat-\Ystar,\Adj-\E\Adj\right\rangle \leq8\sqrt{2\inprob(\num^{3}-\num^{2})}+\frac{32}{3}\num;
\]
whence 
\begin{align*}
\error & \leq\frac{2}{\inprob-\outprob}\left(8\sqrt{2\inprob(\num^{3}-\num^{2})}+\frac{32}{3}\num\right)\overset{(i)}{\le}\frac{45\sqrt{\inprob\num^{3}}}{\inprob-\outprob}=45\sqrt{\frac{\num^{3}}{\snr}},
\end{align*}
 where step $(i)$ holds by our assumption $\inprob\geq1/\num$. 

\section{Technical lemmas}

In this section we provide the proofs of the technical lemmas used
in the main text.

\subsection{Proof of Fact~\ref{fact:expectation_gamma}\label{sec:proof_fact_expectation_gamma}}

For each pair $i\neq j$, if nodes $i$ and $j$ belong to the same
cluster, then $\ystar_{ij}=1\ge\yhat_{ij}$ and $\E\adj_{ij}-\frac{\inprob+\outprob}{2}=p-\frac{p+q}{2}=\frac{\inprob-\outprob}{2}$;
if nodes $i$ and $j$ belong to different clusters, then $\ystar_{ij}=0\le\yhat_{ij}$
and $\E\adj_{ij}-\frac{\inprob+\outprob}{2}=q-\frac{p+q}{2}=-\frac{\inprob-\outprob}{2}$.
For $i=j$, we have $\yhat_{ij}=\ystar_{ij}$. In each case, we have
the expression 
\[
\left(\ystar_{ij}-\yhat_{ij}\right)\left(\E\adj_{ij}-\frac{\inprob+\outprob}{2}\right)=\frac{\inprob-\outprob}{2}\left|\widehat{Y}_{ij}-\ystar_{ij}\right|.
\]
Summing up the above equation for all $i,j\in[\num]$ gives the desired
equation (\ref{eq:expectation_gamma}).

\subsection{Proof of Lemma \ref{lm:order_stats} \label{sec:proof_lm_order_stat}}

To establish a uniform bound in $\beta$, we apply a discretization
argument to the possible values of $\beta$. Define the shorthand
$E:=(e^{-g\rho/\conste}m,\left\lceil \constu m\right\rceil ]$. We
can cover $E$ by the sub-intervals $E_{t}\coloneqq(t-1,t]$ for integers
$t\in[\left\lceil e^{-g\rho/\conste}m\right\rceil ,\left\lceil \constu m\right\rceil ]$.
By construction we have 
\begin{equation}
\frac{g\rho}{\conste}=\log\frac{m}{e^{-g\rho/\conste}m}\ge\log\frac{m}{t}.\label{eq:requirement_bernoulli_bernstein}
\end{equation}
Also know that our choice of $D=\sqrt{7}$ satisfies $\frac{1}{2}D^{2}\ge3\left(1+\frac{1}{3\sqrt{\conste}}D\right)$
since $\conste\ge28$.

For each $t\in[\left\lceil \constu m\right\rceil ]$ we define the
probabilities 
\[
\alpha_{t}\coloneqq\P\left\{ \exists\beta\in E_{t}:\sum_{j=1}^{\left\lceil \beta\right\rceil }\left|X_{(j)}\right|>D\left\lceil \beta\right\rceil \sqrt{g\rho\log(m/\beta)}\right\} .
\]
We bound each of these probabilities: 
\begin{align}
\alpha_{t} & \overset{(i)}{\le}\P\left\{ \sum_{j=1}^{t}\left|X_{(j)}\right|>Dt\sqrt{g\rho\log(m/t)}\right\} \nonumber \\
 & =\P\left\{ \bigcup_{i_{1}<\cdots<i_{t}}\left\{ \sum_{j=1}^{t}\left|X_{i_{j}}\right|>Dt\sqrt{g\rho\log(m/t)}\right\} \right\} \nonumber \\
 & \leq\sum_{i_{1}<\cdots<i_{t}}\P\left\{ \sum_{j=1}^{t}\left|X_{i_{j}}\right|>Dt\sqrt{g\rho\log(m/t)}\right\} \nonumber \\
 & =\sum_{i_{1}<\cdots<i_{t}}\P\left\{ \max_{\mathbf{u}\in\left\{ \pm1\right\} ^{t}}\sum_{j=1}^{t}X_{i_{j}}u_{j}>Dt\sqrt{g\rho\log(m/t)}\right\} \nonumber \\
 & \leq\sum_{i_{1}<\cdots<i_{t}}\sum_{\mathbf{u}\in\left\{ \pm1\right\} ^{t}}\P\left\{ \sum_{j=1}^{t}X_{i_{j}}u_{j}>Dt\sqrt{g\rho\log(m/t)}\right\} ,\label{eq:double union bd on binom}
\end{align}
where step $(i)$ holds since $\beta\in E_{t}$ implies $\beta\le\left\lceil \beta\right\rceil =t$. 

For each positive integer $t$, fix an index set $(i_{j})_{j=1}^{t}$
and a sign pattern $\mathbf{u}\in\left\{ \pm1\right\} ^{t}$. Note
that $\sum_{j=1}^{t}X_{i_{j}}u_{j}$ is the sum of $tg$ centered
Bernoulli random variables with $\Var(B_{ij}-\E B_{ij})\leq\rho$
and $\left|B_{ij}-\E B_{ij}\right|\leq1$ for each $j\in[t]$ and
$i\in[g]$. We apply Bernstein inequality to bound the probability
on the RHS of equation (\ref{eq:double union bd on binom}): 
\begin{align*}
\P\left\{ \sum_{j=1}^{t}X_{i_{j}}u_{j}>Dt\sqrt{g\rho\log(m/t)}\right\}  & \leq\exp\left\{ -\frac{\frac{1}{2}D^{2}t^{2}g\rho\log(m/t)}{tg\rho+\frac{1}{3}Dt\sqrt{g\rho\log(m/t)}}\right\} \\
 & \overset{(a)}{\leq}\exp\left\{ -\frac{\frac{1}{2}D^{2}t^{2}g\rho\log(m/t)}{\left(1+\frac{1}{3\sqrt{\conste}}D\right)tg\rho}\right\} \\
 & \overset{(b)}{\leq}\exp\left\{ -3t\log(m/t)\right\} ,
\end{align*}
where step (a) holds by equation~(\ref{eq:requirement_bernoulli_bernstein})
and step (b) holds by our choice of the constant $D$. Plugging this
back to equation~(\ref{eq:double union bd on binom}), we get that
for each $\left\lceil e^{-g\rho/\conste}m\right\rceil \leq t\le\left\lceil \constu m\right\rceil $,
\begin{align}
\alpha_{t} & \leq\sum_{i_{1}<\cdots<i_{t}}\sum_{\mathbf{u}\in\left\{ \pm1\right\} ^{t}}\exp\left\{ -3t\log(m/t)\right\} \nonumber \\
 & =\binom{m}{t}\cdot2^{t}\cdot\exp\left\{ -3t\log(m/t)\right\} \nonumber \\
 & \leq\left(\frac{me}{t}\right)^{t}e^{t}\exp\left\{ -3t\log(m/t)\right\} \nonumber \\
 & =\exp\left\{ t\log(m/t)+2t-3t\log(m/t)\right\} \nonumber \\
 & \leq\exp\left\{ -t\log(m/t)\right\} =\left(\frac{t}{m}\right)^{t},\label{eq:binom coeff bound}
\end{align}
where the last inequality follows from $\log(m/t)\ge\log(m/(\constu m+1))\ge2$
since $m\ge8$ and $\constu$ is sufficiently small. It follows that
\begin{align*}
 & \quad\P\left\{ \exists\beta\in E:\sum_{j=1}^{\left\lceil \beta\right\rceil }\left|X_{(j)}\right|>D\left\lceil \beta\right\rceil \sqrt{g\rho\log(m/\beta)}\right\} \\
 & \leq\P\left\{ \bigcup_{t=\left\lceil me^{-\rho g/\conste}\right\rceil }^{\left\lceil \constu m\right\rceil }\left\{ \exists\beta\in E_{t}:\sum_{j=1}^{\left\lceil \beta\right\rceil }\left|X_{(j)}\right|>D\left\lceil \beta\right\rceil \sqrt{g\rho\log(m/\beta)}\right\} \right\} \\
 & \le\sum_{t=\left\lceil me^{-\rho g/\conste}\right\rceil }^{\left\lceil \constu m\right\rceil }\alpha_{t}\\
 & \le\sum_{t=\left\lceil me^{-\rho g/\conste}\right\rceil }^{\left\lceil \constu m\right\rceil }\left(\frac{t}{m}\right)^{t}\\
 & \leq\sum_{t=1}^{\left\lceil \constu m\right\rceil }\left(\frac{t}{m}\right)^{t}\eqqcolon P_{1}(m).
\end{align*}

It remains to show that $P_{1}(m)\leq\frac{3}{m}$. First suppose
that $m\ge16$. Since 
\[
P_{1}(m)\leq\frac{1}{m}+\frac{4}{m^{2}}+\frac{27}{m^{3}}+\sum_{t=4}^{\left\lceil \constu m\right\rceil }\left(\frac{t}{m}\right)^{t}\le\frac{1.5}{m}+m\cdot\max_{t=4,5,\ldots,\left\lceil \constu m\right\rceil }\left(\frac{t}{m}\right)^{t},
\]
the proof is completed if for each integer $t=4,5,\ldots,\left\lceil \constu m\right\rceil $,
we can show the bound $\left(\frac{t}{m}\right)^{t}\leq\frac{1}{m^{2}}$,
or equivalently $f(t):=t(\log m-\log t)\geq2\log m.$ Since $m\ge16$,
we must have $t\le\left\lceil \constu m\right\rceil \le\constu m+1\le\frac{m}{3}$,
in which case $f(t)$ has derivative 
\[
f'(t)=\log m-\log t-1\ge\log3-1\ge0.
\]
Therefore, $f(t)$ is non-decreasing for $4\le t\le\left\lceil \constu m\right\rceil $
and hence $f(t)\ge f(4)=4\log m-4\log4\ge2\log m.$ Now suppose that
$m\le16$. In this case, we have $\left\lceil \constu m\right\rceil =1$,
so $t$ can only be $1$ and hence $P_{1}(m)\le\frac{3}{m}$. $\qed$

\subsection{Proof of Lemma \ref{lm:spectral_bound_A}\label{sec:proof_spectral_bound_A}}

Our proof follows similar lines as that of Lemma 3 in \cite{CMM}.
We first bound $\E\opnorm{\Adj-\E\Adj}$ using a standard symmetrization
argument. Let $\Adj'$ be an independent copy of $\Adj$, and $\mathbf{R}$
be an $\num\times\num$ symmetric matrix, independent of both $\Adj$
and $\Adj'$, with i.i.d.\ Rademacher entries on and above its diagonal.
We can compute 
\begin{align*}
\E\opnorm{\Adj-\E\Adj} & =\E\opnorm{\Adj-\E\Adj'}\\
 & \overset{(i)}{\le}\E\opnorm{\Adj-\Adj'}\\
 & \overset{(ii)}{=}\E\opnorm{\left(\Adj-\Adj'\right)\circ\mathbf{R}}\\
 & \overset{(iii)}{\le}2\E\opnorm{\Adj\circ\mathbf{R}},
\end{align*}
where step $(i)$ follows from convexity of the spectral norm, step
$(ii)$ holds since the matrices $\Adj-\Adj'$ and $\left(\Adj-\Adj'\right)\circ\mathbf{R}$
are identically distributed, and step $(iii)$ follows from the triangle
inequality.

We proceed by bounding $\E\opnorm{\Adj\circ\mathbf{R}}$. Write $\left\Vert \zeta\right\Vert _{a}:=\E[\zeta^{a}]^{1/a}$
for a random variable $\zeta$ and an integer $a\ge0$. Define the
scalars $(b_{ij})_{i\ge j}$ by $b_{ij}=\sqrt{\E\adj_{ij}}$. Also
define the independent random variables $(\xi_{ij})_{i\ge j}$ as
follows: if $\E\adj_{ij}=0$, then $\xi_{ij}$ is a Rademacher random
variable; if $\E\adj_{ij}>0$, 
\[
\xi_{ij}=\begin{cases}
\frac{1}{\sqrt{\E\adj_{ij}}}, & \mbox{with \ensuremath{\inprob}robability \ensuremath{\E\adj_{ij}/2}},\\
-\frac{1}{\sqrt{\E\adj_{ij}}}, & \mbox{with probability \ensuremath{\E\adj_{ij}/2}},\\
0, & \mbox{with probability \ensuremath{1-\E\adj_{ij}}.}
\end{cases}
\]
It can be seen that the lower triangular entries of the symmetric
matrix $\Adj\circ\mathbf{R}$ can be written as $\adj_{ij}R_{ij}=\xi_{ij}b_{ij}$.
By definition, the random variables $(\xi_{ij})_{i\ge j}$ are symmetric
and have zero mean and unit variance. We can therefore make use of
the following known result:
\begin{lem}[Corollary 3.6 in \cite{bandeira2016}]
Let $\X$ be the $\num\times\num$ symmetric random matrix with $X_{ij}=\xi_{ij}b_{ij}$,
where $\{\xi_{ij},i\geq j\}$ are independent symmetric random variables
with unit variance and $\{b_{ij}:i\geq j\}$ are given scalars. Then
we have for any $\alpha\geq3$, 
\[
\E\opnorm{\X}\leq e^{2/\alpha}\left\{ 2\sigma+14\alpha\max_{ij}\|\xi_{ij}b_{ij}\|_{2\left\lceil \alpha\log\num\right\rceil }\sqrt{\log\num}\right\} ,
\]
where $\sigma\coloneqq\max_{i}\sqrt{\sum_{j}b_{ij}^{2}}$. 
\end{lem}
For every pair $i\ge j$, the random variable $\xi_{ij}b_{ij}$ is
surely bounded by $1$ in absolute value and thus satisfies $\|\xi_{ij}b_{ij}\|_{2\left\lceil \alpha\log\num\right\rceil }\le1$
for any $\alpha\ge3$. The scalars $(b_{ij})_{i\ge j}$ are all bounded
by $\sqrt{\inprob}$, so $\sigma\leq\sqrt{\inprob\num}$. Applying
the above lemma with $\alpha=3$ gives 
\begin{align*}
\E\opnorm{\Adj\circ\mathbf{R}} & \leq e^{2/3}\left(2\sqrt{\inprob\num}+42\sqrt{\log\num}\right)\leq4\sqrt{\inprob\num}+84\sqrt{\log\num},
\end{align*}
whence
\[
\E\opnorm{\Adj-\E\Adj}\le2\E\opnorm{\Adj\circ\mathbf{R}}\le8\sqrt{\inprob\num}+168\sqrt{\log\num}.
\]

We complete the proof by bounding the deviation of $\opnorm{\Adj-\E\Adj}$
from its expectation. This can be done using a standard Lipschitz
concentration inequality:
\begin{lem}[Theorem 6.10 in \cite{boucheron2013concentration}, generalized in
Exercise 6.5 therein]
Let ${\cal X}\subset\mathbb{R}^{d}$ be a convex compact set with
diameter $B$. Let $X_{1},\ldots,X_{N}$ be independent random variables
taking values in ${\cal X}$ and assume that $f:{\cal X}^{N}\to\mathbb{R}$
is separately convex and $1$-Lipschitz, that is, $|f(x)-f(y)|\leq\|x-y\|$
for all $x,y\in{\cal X}^{N}\subset\mathbb{R}^{dN}$. Then $Z=f(X_{1},\ldots,X_{N})$
satisfies, for all $t>0$, 
\[
\P\left\{ Z>\E Z+t\right\} \leq e^{-t^{2}/(2B^{2})}.
\]
\end{lem}
To use this result for our purpose, we note that $Z=\opnorm{\Adj-\E\Adj}/\sqrt{2}$
is a function of the $N=\num(\num-1)/2$ independent lower triangular
entries of the symmetric matrix $\Adj-\E\Adj$. Moreover, this function
is separately convex and $1$-Lipschitz. In our setting, each entry
of $\Adj-\E\Adj$ takes values in the interval ${\cal X}=[-\inprob,1-\outprob]$
, which is convex compact with diameter $B=1-\outprob+\inprob\leq2$.
Applying the above lemma yields that for each $t\geq0$, 
\begin{align*}
\P\left\{ \opnorm{\Adj-\E\Adj}/\sqrt{2}>\E\opnorm{\Adj-\E\Adj}/\sqrt{2}+t\right\}  & \leq e^{-t^{2}/(2B^{2})}\leq e^{-t^{2}/8}.
\end{align*}
Choosing $t=3\sqrt{2\log\num}$ and combining with the previous bound
on $\E\opnorm{\Adj-\E\Adj}$, gives the desired inequality
\[
\P\left\{ \opnorm{\Adj-\E\Adj}>8\sqrt{\inprob\num}+174\sqrt{\log\num}\right\} \leq\num^{-2}.
\]

\subsection{Proof of Lemma \ref{lm:trimmed_spectral_bound}\label{sec:proof_trimmed_matrix_bound}}

For future reference, we state below a more general result; Lemma
\ref{lm:trimmed_spectral_bound} is a special case with $\sigma^{2}=\inprob$.
\begin{lem}
Suppose that $\X\in\real^{\num\times\num}$ is a random matrix with
independent entries of the following distributions 
\[
X_{ij}=\begin{cases}
1-\inprob_{ij} & \mbox{with probability }\inprob_{ij}\\
-\inprob_{ij} & \mbox{with probability }1-\inprob_{ij}.
\end{cases}
\]
Let $\sigma$ be a number that satisfies $\inprob_{ij}\leq\sigma^{2}$
for all $(i,j)$, and $\overleftrightarrow{\X}$ be the matrix obtained
from $\X$ by zeroing out all the rows and columns having more than
$40\sigma^{2}\num$ positive entries. Then with probability at least
$1-P_{3}=1-e^{-(3-\ln2)2\num}-(2\num)^{-3},$ $\opnorm{\overleftrightarrow{\X}}\leq C\sigma\sqrt{\num}$
for some absolute constant $C>0$.
\end{lem}
We now prove this lemma. Note that the matrix $\X$ is not symmetric.
To make use of a known result that requires symmetry, we employ a
standard dilation argument. We construct the matrix 
\[
\D\coloneqq\begin{bmatrix}\O_{1} & \X\\
\X^{\top} & \O_{2}
\end{bmatrix}\in\real^{2\num\times2\num},
\]
where $\O_{1},\O_{2}\in\real^{\num\times\num}$ are all-zero matrices.
The matrices $\O_{1}$ and $\O_{2}$ can be viewed as random symmetric
matrices whose entries above the diagonal are independent centered
Bernoulli random variable of rate zero. The matrix $\D$ is symmetric
and satisfies the assumptions in the following known result with $N=2\num$.
\begin{lem}[Lemma 12 of \cite{ChinRaoVu15}]
\label{lm:ChinRaoVu}Suppose that $\D\in\real^{N\times N}$ is a
random symmetric matrix with zero on the diagonal whose entries above
the diagonal are independent with the following distributions 
\[
D_{ij}=\begin{cases}
1-\inprob_{ij} & \mbox{with probability }\inprob_{ij}\\
-\inprob_{ij} & \mbox{with probability }1-\inprob_{ij}.
\end{cases}
\]
Let $\sigma$ be a quantity such that $\inprob_{ij}\leq\sigma^{2}$
for all $(i,j)\in[N]\times[N]$, and $\D_{1}$ be the matrix obtained
from $\D$ by zeroing out all the rows and columns having more than
$20\sigma^{2}N$ positive entries. Then with probability $1-o(1),$
 $\opnorm{\D_{1}}\leq C'\sigma\sqrt{N}$ for some absolute constant
$C'>0$.
\end{lem}
\begin{rem*}
The probability above is in fact $1-o(1)=1-P_{3}=1-e^{-(3-\ln2)N}-N^{-3}$,
as can be seen by inspecting the proof of the lemma. Moreover, since
$\D$ is symmetric, the indices of the rows and columns being zeroed
out are identical.

Lemma~\ref{lm:ChinRaoVu} ensures that $\opnorm{\D_{1}}\leq C'\sigma\sqrt{2\num}=C\sigma\sqrt{\num}$
with probability at least $1-P_{3}$, where $C=\sqrt{2}C'$. The lemma
at the beginning of this sub-section then follows from the claim that
$\opnorm{\overleftrightarrow{\X}}=\opnorm{\D_{1}}.$

It remains to prove the above claim. Recall that $\pos{}_{i\bullet}(\M)$
and $\pos{}_{\bullet j}(\M)$ are the numbers of positive entries
in the $i$-th row and $j$-th column of a matrix $\M$, respectively.
For each $(i,j)\in[\num]\times[\num]$, we observe that by construction,
\begin{align*}
\pos{}_{i\bullet}(\D)>20\sigma^{2}N & \;\Leftrightarrow\;\pos_{i\bullet}(\X)>40\sigma^{2}\num\\
\pos_{\bullet\left(j+\num\right)}(\D)>20\sigma^{2}N & \;\Leftrightarrow\;\pos_{\bullet j}(\X)>40\sigma^{2}\num
\end{align*}
and similarly
\begin{align*}
\pos_{\left(i+\num\right)\bullet}(\D)>20\sigma^{2}N & \;\Leftrightarrow\;\pos_{i\bullet}(\X^{\top})>40\sigma^{2}\num\\
\pos_{\bullet j}(\D)>20\sigma^{2}N & \;\Leftrightarrow\;\pos_{\bullet j}(\X^{\top})>40\sigma^{2}\num.
\end{align*}
It then follows from the definitions of $\D$ and $\D_{1}$ that 
\[
\D_{1}=\begin{bmatrix}\O_{1} & \overleftrightarrow{\X}\\
\overleftrightarrow{\X^{\top}} & \O_{2}
\end{bmatrix}=\begin{bmatrix}\O_{1} & \overleftrightarrow{\X}\\
\overleftrightarrow{\X}{}^{\top} & \O_{2}
\end{bmatrix},
\]
whence $\opnorm{\D_{1}}=\opnorm{\overleftrightarrow{\X}}$. 
\end{rem*}

\subsection{Proof of Fact \ref{fact:Yhp_infty}\label{sec:proof_Yhp_infty}}

Because $\Yhat$ is feasible to the SDP (\ref{eq:SDP1}), we know
that $\norm[\Yhat]{\infty}\le1$. Let $c(i)$ be the index of the
cluster that contains node $i$. For each $(i,j)\in[\num]\times[\num]$,
we have the bound
\[
\vert(\U\U^{\top}\Yhat)_{ij}\vert=\left|\sum_{t:c(t)=c(i)}(\U\U^{\top})_{it}\yhat_{tj}\right|\le\size\cdot\frac{1}{\size}\cdot\norm[\Yhat]{\infty}\le1
\]
thanks to the structure of the matrix $\U\U^{\top}$. The same bound
holds for the matrices $\Yhat\U\U^{\top}$ and $\U\U^{\top}\Yhat\U\U^{\top}$
by similar arguments. It follows that 
\begin{align*}
\norm[\PTperp(\Yhat)]{\infty} & =\norm[\Yhat-\U\U^{\top}\Yhat-\Yhat\U\U^{\top}+\U\U^{\top}\Yhat\U\U^{\top}]{\infty}\\
 & \leq\norm[\Yhat]{\infty}+\norm[\U\U^{\top}\Yhat]{\infty}+\norm[\Yhat\U\U^{\top}]{\infty}+\norm[\U\U^{\top}\Yhat\U\U^{\top}]{\infty}\\
 & \le4.
\end{align*}

\subsection{Proof of the implication (\ref{eq:sub-exp-tail})$\Rightarrow$(\ref{eq:sub-exp-mgf})\label{sec:proof_sub_exp_eqv_defs}}

We first show that the tail condition~(\ref{eq:sub-exp-tail}) of
a random variable $X$ implies a bound for its moments. Note that
we do not require $X$ to be centered. 
\begin{lem}
\label{lm:subexp_tai_to_moment} Suppose that for some $\lambda>0$,
the random variable $X$ satisfies $\P\left(\left|X\right|>z\right)\leq2e^{-8z/\lambda},\forall z\ge0$
. Then for each positive integer $m\geq1$, 
\[
\left(\E\left[\left|X\right|^{m}\right]\right)^{1/m}\leq\frac{1}{4}\lambda\left(m!\right)^{1/m}.
\]
\end{lem}
\begin{proof}
Let $\Gamma(\cdot)$ denote the Gamma function. We have the bound
\begin{align*}
\E\left[\left|X\right|^{m}\right] & =\int_{0}^{\infty}\P\left(\left|X\right|^{m}>z\right)\diff z\\
 & =\int_{0}^{\infty}\P\left(\left|X\right|>z^{1/m}\right)\diff z\\
 & \leq\int_{0}^{\infty}2\exp\left(-\frac{8z^{1/m}}{\lambda}\right)\diff z\\
 & \overset{(i)}{=}2\cdot\left(\lambda/8\right)^{m}\cdot m\int_{0}^{\infty}e^{-u}u^{m-1}\diff u\\
 & \leq(\lambda/4)^{m}\cdot m\Gamma\left(m\right)=(\lambda/4)^{m}m!,
\end{align*}
where step $(i)$ follows from a change of variable $u=8z^{1/m}/\lambda$.
Taking the $m$-th root of both sides proves the result.
\end{proof}
Using Minkowski's inequality, we can see that the above moment bounds
are sub-additive:
\begin{lem}
\label{lm:subexp_moment_subaddtive} Suppose that for some $\lambda_{1},\lambda_{2}>0$,
the random variables $X_{1}$ and $X_{2}$ satisfy 
\begin{align*}
\left(\E\left[\left|X_{1}\right|^{m}\right]\right)^{1/m} & \leq\lambda_{1}\left(m!\right)^{1/m}\quad\text{and}\quad\left(\E\left[\left|X_{2}\right|^{m}\right]\right)^{1/m}\leq\lambda_{2}\left(m!\right)^{1/m}
\end{align*}
for each positive integer $m$. Then for each positive integer $m$,
\[
\left(\E\left[\left|X_{1}+X_{2}\right|^{m}\right]\right)^{1/m}\leq\left(\lambda_{1}+\lambda_{2}\right)\left(m!\right)^{1/m}.
\]
\end{lem}
Consequently, centering a random variable does not affect its sub-exponentiality
up to constant factors. In particular, suppose that $X$ satisfies
the moment bounds in Lemma~\ref{lm:subexp_tai_to_moment}. Then for
every positive integer $m$, 
\[
\left(\E\left[\left|\E X\right|^{m}\right]\right)^{1/m}=\left(\left|\E X\right|^{m}\right)^{1/m}\overset{(i)}{\le}\left(\E\left|X\right|^{m}\right)^{1/m}\leq\frac{\lambda}{4}\left(m!\right)^{1/m},
\]
where step $(i)$ uses the Jensen's inequality. Applying Lemma \ref{lm:subexp_moment_subaddtive}
gives 
\[
\left(\E\left[\left|X-\E X\right|^{m}\right]\right)^{1/m}\leq\frac{\lambda}{2}\left(m!\right)^{1/m}.
\]

The next lemma shows that the moment bound implies the bound~(\ref{eq:sub-exp-mgf})
on the moment generating function, hence completing the proof of the
implication (\ref{eq:sub-exp-tail})$\Rightarrow$(\ref{eq:sub-exp-mgf}).
\begin{lem}
\label{lm: subexp_moment_to_mgf} Suppose that for some number $\lambda>0$,
the random variable $X$ satisfies 
\[
\left(\E\left[\left|X-\E X\right|^{m}\right]\right)^{1/m}\leq\frac{\lambda}{2}\cdot\left(m!\right)^{1/m}
\]
for each positive integer $m$. Then we have 
\[
\E\left[e^{t(X-\E X)}\right]\leq e^{t^{2}\lambda^{2}/2},\qquad\forall\left|t\right|\leq\frac{1}{\lambda}.
\]
\end{lem}
\begin{proof}
For each $t$ such that $\left|t\right|\leq1/\lambda$, we have 
\begin{align*}
\E\left[e^{t(X-\E X)}\right] & =1+t\E\left[X-\E X\right]+\sum_{m=2}^{\infty}\frac{t^{m}\E\left[(X-\E X)^{m}\right]}{m!}\\
 & \overset{(i)}{\leq}1+\sum_{m=2}^{\infty}\frac{\left|t\right|^{m}\E\left[\left|X-\E X\right|^{m}\right]}{m!}\\
 & \leq1+\sum_{m=2}^{\infty}\left(\left|t\right|\cdot\frac{\lambda}{2}\right)^{m}\\
 & =1+\frac{t^{2}\lambda^{2}}{4}\sum_{m=0}^{\infty}\left(\left|t\right|\cdot\frac{\lambda}{2}\right)^{m}\\
 & \leq1+\frac{t^{2}\lambda^{2}}{2}\leq e^{t^{2}\lambda^{2}/2},
\end{align*}
where step $\left(i\right)$ holds since $\E\left[X-\E X\right]=0$.
\end{proof}

\section{Proof of Theorem \ref{thm:cluster_error_rate}\label{sec:proof_cluster_error_rate}}

As mentioned in Appendix~\ref{sec:kmedians}, we use an approximate
$k$-medians clustering algorithm with approximation ratio $\apxconst=\frac{20}{3}$.
Theorem \ref{thm:cluster_error_rate} follows immediately by combining
Theorem \ref{thm:exp_rate} and the proposition below.
\begin{prop}
\label{prop:kmedian}The clustering assignment $\LabelHat=\apkmedian(\Yhat)$
produced by the $\apxconst$-approximate $k$-median algorithm (Algorithm~\ref{alg:apx_kmedian})
satisfies the error bound 
\[
\misrate(\LabelHat,\LabelStar)\leq2(1+2\apxconst)\cdot\frac{\norm[\Yhat-\Ystar]1}{\norm[\Ystar]1}.
\]
\end{prop}
We note that a similar result appeared in the paper \cite{CMM}, specifically
their Theorem 2 restricted to the non-degree-corrected setting. The
proposition provides a moderate generalization, establishing a general
relationship between clustering error of the $\apxconst$-approximation
$k$-median procedure and the error of its input $\Yhat$. The rest
of this section is devoted to the proof of Proposition \ref{prop:kmedian}.
\\

Recall that $(\overline{\boldsymbol{\Psi}},\overline{\mathbf{X}})$
is the optimal solution of the $k$-medians problem~(\ref{eq:k-median problem}),
and $(\check{\boldsymbol{\Psi}},\check{\mathbf{X}})$ is a $\apxconst$-approximate
solution. Set $\overline{\mathbf{Y}}\coloneqq\overline{\boldsymbol{\Psi}}\overline{\mathbf{X}}$
and $\check{\mathbf{Y}}\coloneqq\check{\boldsymbol{\Psi}}\check{\mathbf{X}}$.
Note that the solution $(\check{\boldsymbol{\Psi}},\check{\mathbf{X}})$
corresponds to at most $\numclust$ clusters. Without loss of generality
we may assume that it actually contains exactly $\numclust$ clusters,
and thus the cluster membership matrix$\check{\boldsymbol{\Psi}}$
is in $\mathbb{M}_{\num,\numclust}$ and has exactly $\numclust$
distinct rows. If this is not true, we can always move an arbitrary
point from the $\num$ input points to a new cluster, without changing
the $k$-medians objective value of the approximation solution $(\check{\boldsymbol{\Psi}},\check{\mathbf{X}})$. 

We next rewrite the cluster error metric (\ref{eq:mis_rate_def})
in matrix form. Let $\boldsymbol{\Psi}^{*}\in\mathbb{M}_{\num,\numclust}$
be the membership matrix corresponding to the ground truth clusters;
that is, for each $i\in[\num]$, $\Psi_{i\labelstar_{i}}^{*}$ is
the only non-zero element of the $i$-th row of $\boldsymbol{\Psi}^{*}$,
and thus $\boldsymbol{\Psi}^{*}(\boldsymbol{\Psi}^{*})^{\top}=\Ystar$.
Let ${\cal S}_{\numclust}$ be the set of $\numclust\times\numclust$
permutation matrices. The set of misclassified nodes with respect
to a permutation $\boldsymbol{\Pi}\in{\cal S}_{\numclust}$, is then
given by
\[
{\cal E}(\boldsymbol{\Pi})\coloneqq\left\{ i\in[\num]:\left(\check{\boldsymbol{\Psi}}\boldsymbol{\Pi}\right)_{i\bullet}\ne\boldsymbol{\Psi}_{i\bullet}^{*}\right\} .
\]
With this notation, the error metric (\ref{eq:mis_rate_def}) can
be expressed as $\misrate(\LabelHat,\LabelStar)=\min_{\boldsymbol{\Pi}\in{\cal S}_{\numclust}}\num^{-1}\left|{\cal E}(\boldsymbol{\Pi})\right|,$
and it remains to bound the RHS.

To this end, we construct several useful sets. For each $a\in[\numclust]$,
let $C_{a}^{*}=\{i\in[\num]:\labelstar_{i}=a\}$ be the $a$-th cluster,
and we define the node index sets 
\[
T_{a}\coloneqq\left\{ i\in C_{a}^{*}:\norm[\check{\mathbf{Y}}_{i\bullet}-\Ystar_{i\bullet}]1<\size\right\} 
\]
and $S_{a}\coloneqq C_{a}^{*}\backslash T_{a}$. Let $T\coloneqq\bigcup_{a\in[\numclust]}T_{a}$
and $S\coloneqq\bigcup_{a\in[\numclust]}S_{a}$. Note that $S_{1},\ldots,S_{\numclust},T_{1},\ldots,T_{\numclust}$
are disjoint with $T\cup S=[\num]$. Further define the cluster index
sets
\begin{align*}
R_{1} & \coloneqq\left\{ a\in[\numclust]:T_{a}=\emptyset\right\} ,\\
R_{2} & \coloneqq\left\{ a\in[\numclust]:T_{a}\ne\emptyset;\check{\boldsymbol{\Psi}}_{i\bullet}=\check{\boldsymbol{\Psi}}_{j\bullet},\forall i,j\in T_{a}\right\} ,\\
R_{3} & \coloneqq[\numclust]\setminus(R_{1}\cup R_{2})=\left\{ a\in[\numclust]:T_{a}\ne\emptyset;\check{\boldsymbol{\Psi}}_{i\bullet}\ne\check{\boldsymbol{\Psi}}_{j\bullet},\exists i,j\in T_{a}\right\} .
\end{align*}
Note that $R_{1},R_{2}$ and $R_{3}$ are disjoint sets. With the
above notations, we have the following claims.
\begin{claim}
\label{clm:misclass<=00003DS+R3}$\min_{\boldsymbol{\Pi}\in{\cal S}_{\numclust}}\left|{\cal E}(\boldsymbol{\Pi})\right|\leq\left|S\right|+\left|R_{3}\right|\size$.
\end{claim}

\begin{claim}
\label{clm:R3<=00003DR1}$\left|R_{3}\right|\leq\left|R_{1}\right|$.
\end{claim}

\begin{claim}
\label{clm:S upper and lower bound}$\left|R_{1}\right|\size\leq\left|S\right|\leq\frac{1}{\size}\norm[\check{\mathbf{Y}}-\Ystar]1$.
\end{claim}

\begin{claim}
\label{clm:Ycheck - Ystar upper bound}$\norm[\check{\mathbf{Y}}-\Ystar]1\leq(1+2\apxconst)\norm[\Yhat-\Ystar]1$.
\end{claim}
Applying the above claims in order, we obtain that 
\begin{align*}
\min_{\boldsymbol{\Pi}\in{\cal S}_{\numclust}}\num^{-1}\left|{\cal E}(\boldsymbol{\Pi})\right| & \le\frac{2(1+2\apxconst)}{\num\size}\norm[\Yhat-\Ystar]1=2(1+2\apxconst)\frac{\norm[\Yhat-\Ystar]1}{\norm[\Ystar]1},
\end{align*}
where the last equality follows from $\norm[\Ystar]1=\num\size$.
Combining with the aforementioned expression $\misrate(\LabelHat,\LabelStar)=\min_{\boldsymbol{\Pi}\in{\cal S}_{\numclust}}\num^{-1}\left|{\cal E}(\boldsymbol{\Pi})\right|$
proves Proposition \ref{prop:kmedian}. 

We prove the the above claims in the sub-sections to follow.

\subsection{Proof of Claim \ref{clm:misclass<=00003DS+R3}}

We record a property of the cluster membership matrix $\check{\boldsymbol{\Psi}}$
of the approximate $k$-medians solution.
\begin{lem}
\label{lm:distinct_rows_across_clusters}For each cluster pair $a,b\in R_{2}\cup R_{3}$
with $a\ne b$ and each node pair $i\in T_{a},j\in T_{b}$, we have
$\check{\boldsymbol{\Psi}}_{i\bullet}\ne\check{\boldsymbol{\Psi}}_{j\bullet}$.
\end{lem}
\begin{proof}
For each pair $a,b\in R_{2}\cup R_{3}$ with $a\ne b$, we have $T_{a}\ne\emptyset$
and $T_{b}\ne\emptyset$ by definition. For each pair $i\in T_{a},j\in T_{b}$,
we have the inequality 
\[
\norm[\check{\mathbf{Y}}_{i\bullet}-\check{\mathbf{Y}}_{j\bullet}]1\geq\norm[\Ystar_{i\bullet}-\Ystar_{j\bullet}]1-\norm[\check{\mathbf{Y}}_{i\bullet}-\Ystar_{i\bullet}]1-\norm[\check{\mathbf{Y}}_{j\bullet}-\Ystar_{j\bullet}]1.
\]
Since nodes $i$ and $j$ are in two different clusters, each of $\Ystar_{i\bullet}$
and $\Ystar_{j\bullet}$ is a binary vector with exactly $\size$
ones, and they have disjoint support; we therefore have $\norm[\Ystar_{i\bullet}-\Ystar_{j\bullet}]1=2\size$.
Moreover, note that $\norm[\check{\mathbf{Y}}_{i\bullet}-\Ystar_{i\bullet}]1<\size$
and $\norm[\check{\mathbf{Y}}_{j\bullet}-\Ystar_{j\bullet}]1<\size$
by definition of $T_{a}$ and $T_{b}$. It follows that $\norm[\check{\mathbf{Y}}_{i\bullet}-\check{\mathbf{Y}}_{j\bullet}]1>2\size-\size-\size=0$,
and thus $\check{\mathbf{Y}}_{i\bullet}\ne\check{\mathbf{Y}}_{j\bullet}$.
The latter implies that $\check{\boldsymbol{\Psi}}_{i\bullet}\ne\check{\boldsymbol{\Psi}}_{j\bullet}$;
otherwise we would reach a contradiction $\check{\mathbf{Y}}_{i\bullet}=\left(\check{\boldsymbol{\Psi}}_{i\bullet}\right)^{\top}\check{\mathbf{X}}=\left(\check{\boldsymbol{\Psi}}_{j\bullet}\right)^{\top}\check{\mathbf{X}}=\check{\mathbf{Y}}_{j\bullet}$. 
\end{proof}
Proceeding to the proof of Claim \ref{clm:misclass<=00003DS+R3},
we observe that $S,T_{1},T_{2},\ldots,T_{\numclust}$ are disjoint
and satisfy
\[
[\num]=S\cup T=S\cup\left(\bigcup_{a\in R_{2}}T_{a}\right)\cup\left(\bigcup_{a\in R_{3}}T_{a}\right),
\]
where the last equality holds since $a\in R_{1}$ implies $T_{a}=\emptyset$.
To prove the claim, it suffices to show that there exists some $\boldsymbol{\Pi}\in{\cal S}_{\numclust}$
such that ${\cal E}(\boldsymbol{\Pi})\subseteq S\cup\left(\bigcup_{a\in R_{3}}T_{a}\right)=[\num]\backslash\bigcup_{a\in R_{2}}T_{a}$.
Indeed, for each $a\in R_{2}$, any pair $i,j\in T_{a}$ satisfies
$\check{\boldsymbol{\Psi}}_{i\bullet}=\check{\boldsymbol{\Psi}}_{j\bullet}$
by definition. This fact, combined with Lemma \ref{lm:distinct_rows_across_clusters},
implies that there exists some $\boldsymbol{\Pi}\in{\cal S}_{\numclust}$
such that $\left(\check{\boldsymbol{\Psi}}\boldsymbol{\Pi}\right)_{i\bullet}=\boldsymbol{\Psi}_{i\bullet}^{*}$
for all $i\in\bigcup_{a\in R_{2}}T_{a}$. By definition of ${\cal E}(\boldsymbol{\Pi})$,
we have ${\cal E}(\boldsymbol{\Pi})\cap\left(\bigcup_{a\in R_{2}}T_{a}\right)=\emptyset$
and are therefore done. 

\subsection{Proof of Claim \ref{clm:R3<=00003DR1}}

The claim follows by rearranging the left and right hand sides of
the equation
\[
\left|R_{2}\right|+2\left|R_{3}\right|\leq\numclust=\left|R_{1}\right|+\left|R_{2}\right|+\left|R_{3}\right|,
\]
which we now prove. The equality follows from the definition of $R_{1},R_{2}$
and $R_{3}$. For the inequality, note that each element of $R_{2}$
contributes at least one distinct row in $\check{\boldsymbol{\Psi}}$
and each element of $R_{3}$ contributes at least two distinct rows
in $\check{\boldsymbol{\Psi}}$. The indices of these rows are all
in $T$ by definition, and Lemma \ref{lm:distinct_rows_across_clusters}
guarantees that these rows are distinct across $R_{2}\cup R_{3}$.
The inequality then follows from the fact that $\check{\boldsymbol{\Psi}}$
has $\numclust$ distinct rows. 

\subsection{Proof of Claim \ref{clm:S upper and lower bound}}

The first inequality in the claim holds because
\begin{align*}
\left|S\right| & =\sum_{a\in[\numclust]}\left|S_{a}\right|\geq\sum_{a\in R_{1}}\left|S_{a}\right|\overset{(i)}{=}\sum_{a\in R_{1}}\left|C_{a}^{*}\right|=\left|R_{1}\right|\size,
\end{align*}
where step $(i)$ holds since $T_{a}=\emptyset$ for each $a\in R_{1}$
and thus $S_{a}=C_{a}^{*}$. On the other hand, we have
\begin{align*}
\left|S\right| & \overset{(ii)}{\leq}\sum_{i\in S}\frac{1}{\size}\norm[\check{\mathbf{Y}}_{i\bullet}-\Ystar_{i\bullet}]1\overset{(iii)}{\le}\frac{1}{\size}\norm[\check{\mathbf{Y}}-\Ystar]1,
\end{align*}
where step $(ii)$ holds since $\frac{1}{\size}\norm[\check{\mathbf{Y}}_{i\bullet}-\Ystar_{i\bullet}]1\geq1$
for each $i\in S$, and step $(iii)$ holds since $S\subseteq[\num]$.
This proves the second inequality in the claim.

\subsection{Proof of Claim \ref{clm:Ycheck - Ystar upper bound}}

Recall that $\Ystar=\boldsymbol{\Psi}^{*}\left(\boldsymbol{\Psi}^{*}\right)^{\top}$
and $\boldsymbol{\Psi}^{*}\in\mathbb{M}_{\num,\numclust}$. Introducing
an extra piece of notation $\mathbf{X}^{*}\coloneqq\left(\boldsymbol{\Psi}^{*}\right)^{\top}$,
we can write $\Ystar=\boldsymbol{\Psi}^{*}\mathbf{X}^{*}$. Let $\widetilde{\mathbf{X}}\in\mathbb{R}^{\numclust\times\num}$
be the matrix whose $a$-th row is equal to the element in $\left\{ \Yhat_{i\bullet}:i\in C_{a}^{*}\right\} $
that is closest to $\mathbf{X}_{a\bullet}^{*}$ in $\size_{1}$ norm;
that is,
\[
\widetilde{\mathbf{X}}_{a\bullet}\coloneqq\argmin_{\mathbf{x}\in\left\{ \Yhat_{i\bullet}:i\in C_{a}^{*}\right\} }\norm[\mathbf{x}-\mathbf{X}_{a\bullet}^{*}]1\quad\text{for each }a\in[\numclust].
\]
Finally let $\widetilde{\mathbf{Y}}\coloneqq\boldsymbol{\Psi}^{*}\widetilde{\mathbf{X}}$.
We have the inequality
\begin{align*}
\norm[\Yhat-\Ystar]1 & =\sum_{a\in[\numclust]}\sum_{i\in C_{a}^{*}}\norm[\Yhat_{i\bullet}-\mathbf{Y}_{i\bullet}^{*}]1\\
 & \overset{(i)}{=}\sum_{a\in[\numclust]}\sum_{i\in C_{a}^{*}}\norm[\Yhat_{i\bullet}-\mathbf{X}_{a\bullet}^{*}]1\\
 & \geq\sum_{a\in[\numclust]}\sum_{i\in C_{a}^{*}}\norm[\widetilde{\mathbf{X}}_{a\bullet}-\mathbf{X}_{a\bullet}^{*}]1\\
 & \overset{(ii)}{=}\sum_{a\in[\numclust]}\sum_{i\in C_{a}^{*}}\norm[\widetilde{\mathbf{Y}}_{i\bullet}-\Ystar_{i\bullet}]1\\
 & =\norm[\widetilde{\mathbf{Y}}-\Ystar]1,
\end{align*}
where step $(i)$ holds since for each $a\in[\numclust]$, the $a$-th
row of $\mathbf{X}^{*}$ is the distinct row in $\Ystar$ corresponding
to the cluster $C_{a}^{*}$ and thus $\mathbf{X}_{a\bullet}^{*}=\Ystar_{i\bullet}$
for all $i\in C_{a}^{*}$; step $(ii)$ can be justified by applying
the same argument to $\widetilde{\mathbf{X}}$ and $\widetilde{\mathbf{Y}}$.
It follows that 
\[
\norm[\widetilde{\mathbf{Y}}-\Yhat]1\leq\norm[\Yhat-\Ystar]1+\norm[\widetilde{\mathbf{Y}}-\Ystar]1\leq2\norm[\Yhat-\Ystar]1
\]
and hence 
\begin{align*}
\norm[\check{\mathbf{Y}}-\Yhat]1 & \overset{(i)}{\leq}\apxconst\norm[\overline{\mathbf{Y}}-\Yhat]1\overset{(ii)}{\leq}\apxconst\norm[\widetilde{\mathbf{Y}}-\Yhat]1\leq2\apxconst\norm[\Yhat-\Ystar]1,
\end{align*}
where step $(i)$ holds by the approximation ratio guarantee (\ref{eq:k-median guarantee}),
and step $(ii)$ holds since $(\overline{\boldsymbol{\Psi}},\overline{\mathbf{X}})$
is optimal for the $k$-medians problem (\ref{eq:k-median problem})
(recall that $\overline{\mathbf{Y}}\coloneqq\overline{\boldsymbol{\Psi}}\overline{\mathbf{X}}$)
while $(\boldsymbol{\Psi}^{*},\widetilde{\mathbf{X}})$ is feasible
for the same problem (because $\boldsymbol{\Psi}^{*}\in\mathbb{M}_{\num,\numclust}$
and by definition the rows of $\widetilde{\mathbf{X}}$ are selected
from $\Yhat$). Combining pieces, we obtain that
\[
\norm[\check{\mathbf{Y}}-\Ystar]1\leq\norm[\Yhat-\Ystar]1+\norm[\check{\mathbf{Y}}-\Yhat]1\leq(1+2\apxconst)\norm[\Yhat-\Ystar]1.
\]

\bibliographystyle{plain}
\bibliography{references}

\end{document}